\documentclass[11pt]{article}
\ifdefined\pdfminorversion\pdfminorversion=7\fi
\usepackage[preprint]{acl}
\usepackage{times}
\usepackage{latexsym}
\usepackage[T1]{fontenc}
\usepackage[utf8]{inputenc}
\usepackage{microtype}
\usepackage{graphicx}
\usepackage{amsmath,amssymb}
\usepackage{amsthm}
\theoremstyle{definition}
\newtheorem{assumption}{Assumption}
\theoremstyle{plain}
\newtheorem{theorem}{Theorem}
\usepackage{booktabs,multirow}
\usepackage{xspace}
\usepackage[table]{xcolor}

\newcommand{\ours}{\textbf{\texttt{SeTPO}}\xspace}
\definecolor{ablationdrop}{HTML}{A94442}
\definecolor{teachergreen}{HTML}{2E8B57}
\definecolor{teacherbluegray}{HTML}{EDF2F7}
\definecolor{oursorange}{HTML}{FFF0DD}
\DeclareRobustCommand{\tablecolorlegend}[1]{%
  \begingroup\setlength{\fboxsep}{0pt}\setlength{\fboxrule}{0.2pt}%
  \raisebox{0.1ex}{\fcolorbox{black!25}{#1}{\rule{0pt}{0.65em}\rule{0.65em}{0pt}}}%
  \endgroup}
\newcommand{\teacherbetter}{\textcolor{teachergreen}{\ensuremath{\boldsymbol{\uparrow}}}}
\newcommand{\ablationchange}[1]{\textcolor{ablationdrop}{(#1)}}

\definecolor{placeholderorange}{HTML}{B84700}
\definecolor{placeholderbg}{HTML}{FFF0D8}

\title{Teacher-Guided Tool Learning for Small Language Models via Bidirectional KL Constraints}
\author{
  Haoxuan Gao\textsuperscript{1},
  Qin-Wen Luo\textsuperscript{1},
  Xiang Chen\textsuperscript{1}\thanks{Corresponding author.},
  Yu Tian\textsuperscript{2} \\
  \bfseries Zequn Sun\textsuperscript{3},
  Xiaofeng Cao\textsuperscript{4},
  Jiangchao Yao\textsuperscript{5},
  Sheng-Jun Huang\textsuperscript{1} \\
  \textsuperscript{1}Nanjing University of Aeronautics and Astronautics \\
  \textsuperscript{2}360 Security Capability Center \quad
  \textsuperscript{3}Nanjing University \\
  \textsuperscript{4}Tongji University \quad
  \textsuperscript{5}Shanghai Jiao Tong University \\
  \texttt{\{haoxuangao0402,xiang\_chen\}@nuaa.edu.cn}
}
\AddToHook{shipout/after}{%
  \ifnum\value{page}=4\relax
    \begingroup\globaldefs=1\raggedbottom\endgroup
  \else\ifnum\value{page}=5\relax
    \begingroup\globaldefs=1\flushbottom\endgroup
  \fi\fi}
\begin{document}
\maketitle
\flushbottom
\predisplaypenalty=0
\begin{abstract}
External tools can extend the capabilities of small language models (SLMs), enabling their deployment in resource-constrained environments. However, their limited capabilities make it difficult for SLMs to discover effective tool-use behaviors through reinforcement learning alone, while direct teacher supervision suffers from teacher--student behavioral distribution mismatch and may propagate erroneous preferences. To address these challenges, we propose Selective Teacher-guided Policy Optimization (\ours), which coordinates student exploration with selective teacher guidance and persistent demonstration supervision. Specifically, \ours incorporates reverse-KL-derived token-level teacher feedback into policy advantages on low-reward trajectories, while auxiliary forward-KL supervision preserves teacher-demonstrated behaviors throughout training. Under an idealized bidirectional-KL objective, we theoretically establish that the optimal policy achieves expected reward no lower than the teacher's while retaining a guaranteed fraction of the teacher's probability mass. Extensive experiments on BFCL-v4 and API-Bank demonstrate that \ours achieves superior overall performance over supervised fine-tuning, reinforcement learning, and distillation baselines across all three student scales.
\end{abstract}
\section{Introduction}
\label{sec:introduction}

\begin{figure}[!t]
\centering
\includegraphics[width=\columnwidth]{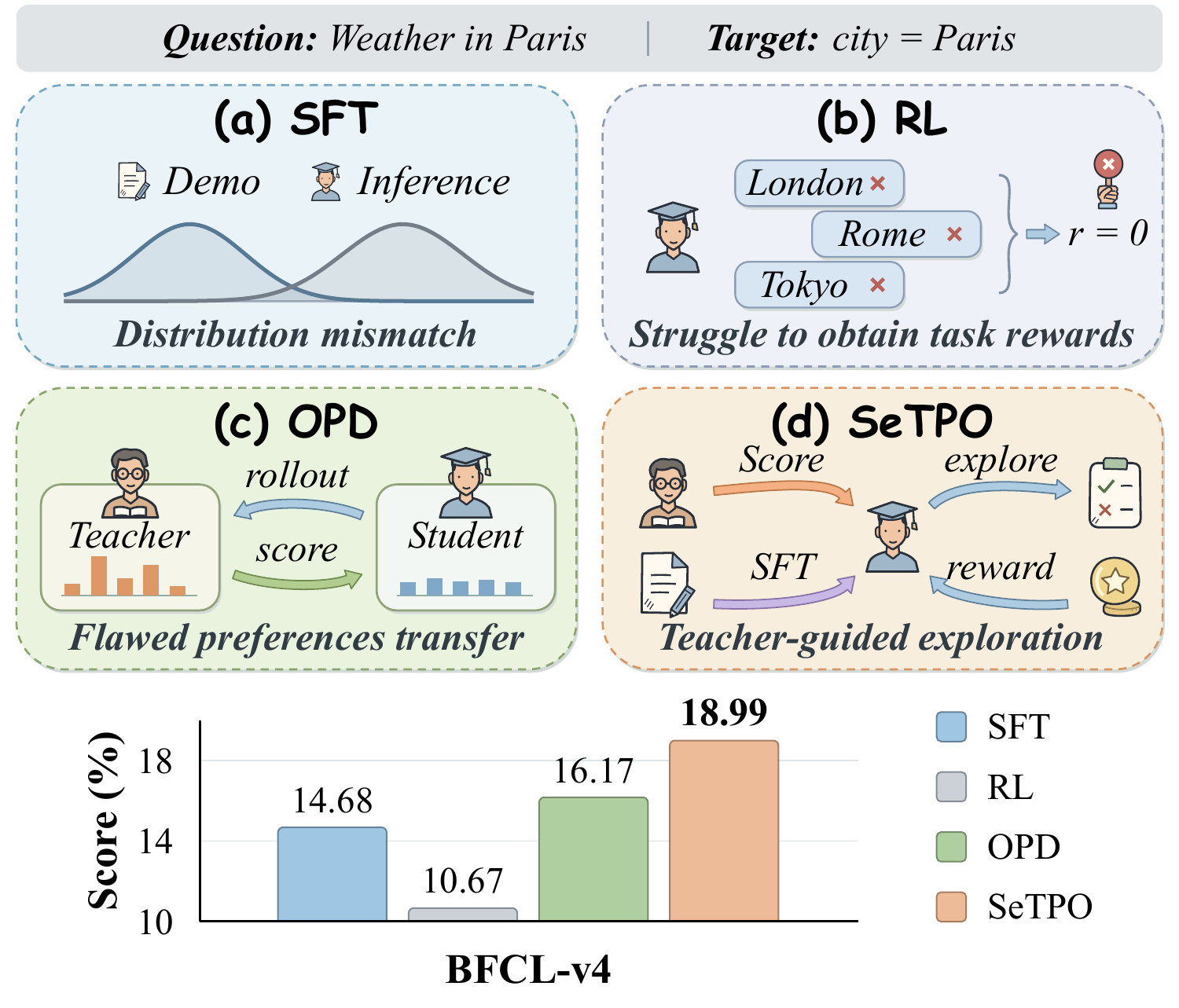}
\caption{Motivation for coordinating supervision in tool learning for small language models. Top: a schematic weather-tool example comparing SFT, GRPO-based RL, the pure OPD baseline used here, and \ours. Bottom: BFCL-v4 scores for Qwen2.5-0.5B.}
\label{fig:motivating_comparison}
\end{figure}

Small language models (SLMs) offer a practical foundation for resource-efficient systems whose capabilities can be extended through external tools \citep{erdogan2024,lin2025}. However, realizing this promise presents a fundamental learning challenge: models with limited underlying capabilities must learn to exploit external resources to accomplish tasks they cannot solve independently. These learning difficulties can prevent SLMs from fully benefiting from external tools. The central question is therefore: \textbf{how can we unlock the full potential of SLMs in tool-use tasks?}

Supervised fine-tuning (SFT) provides an initial foundation through tool-use demonstrations \citep{liu2025}. However, a mismatch between demonstration-based training and the model's own generation at inference can limit its effectiveness \citep{bengio2015,agarwal2024}. To reduce this mismatch, reinforcement learning (RL) optimizes the model's own tool-use attempts using task rewards \citep{qian2025}. Yet SLMs may struggle to discover successful behaviors, leaving few useful experiences for learning. Recent work on on-policy distillation (OPD) supplies additional learning signals through teacher feedback on student-generated trajectories \citep{agarwal2024,gu2024}. However, teacher-preferred outputs may also be incorrect, and imitation can pass these errors to the student \citep{agarwal2024}. These limitations, illustrated in Figure~\ref{fig:motivating_comparison}, raise a more specific question: \textbf{how can SLMs improve their tool-use capabilities through their own exploration with teacher guidance?}

Teacher guidance may not translate into effective learning for SLMs due to teacher--student differences in capacity and behavioral preferences \citep{gu2024,lyu2025}. Simply reproducing teacher preferences therefore does not ensure effective knowledge transfer. This calls for RL to complement teacher guidance with students' own exploration. On-policy RL supports students' own exploration while favoring solutions close to their original behavior distributions \citep{shenfeld2026}. KL constraints provide a way to incorporate teacher guidance into this exploration. This motivates selective teacher-guided RL that supports learning where autonomous exploration falls short, helping unlock SLMs' tool-use potential.

Following this principle, we propose \textbf{Se}lective \textbf{T}eacher-guided \textbf{P}olicy \textbf{O}ptimization (\ours). First, students explore through RL and receive guidance on low-reward trajectories from a frozen teacher, which scores generated tokens under the students' own prefixes. This provides reverse-KL guidance for improving students' own attempts. Second, auxiliary SFT remains active throughout RL, using complete teacher demonstrations to supply behaviors beyond current exploration and maintain the response structure required for tool interaction. This provides forward-KL supervision for retaining teacher-supported behaviors. Under an idealized bidirectional-KL objective, we prove that the optimal policy achieves expected reward no lower than the teacher's while retaining a minimum fraction of the teacher's probability mass on any response set.

Our contributions are:
\begin{itemize}
\item We identify three obstacles to tool learning in SLMs: distribution mismatch between SFT demonstrations and student-generated attempts, limited successful exploration under RL, and teacher--student differences that hinder knowledge transfer in OPD.
\item We propose \ours with selective reverse-KL teacher guidance and persistent forward-KL demonstration supervision. Our idealized bidirectional-KL analysis establishes, at optimality, a lower bound on teacher-behavior retention and expected reward non-degradation relative to the teacher.
\item Experiments on BFCL-v4 and API-Bank show that \ours improves overall tool-use performance across three student scales. It also reduces tool-selection and argument-construction errors and helps maintain stable output formats during training.
\end{itemize}

\section{Related Work}
\label{sec:related_work}

\paragraph{SLM Tool Learning.}
Small language models must learn tool selection, argument construction, and coordinated calls, motivating both supervised demonstrations and autonomous exploration. SFT methods use specialized or verified synthetic demonstrations to teach these behaviors, including multi-turn interactions \citep{erdogan2024,liu2024,zhang2025,prabhakar2025}. Capacity-aware data generation and name masking further adapt supervision to tool learning \citep{liu2025,lin2025}. RL methods optimize models' attempts with task-, call-, or step-level rewards that assess calling format, functional correctness, and task contribution \citep{qian2025,yu2025,li2026}, or develop exploration through reasoning without annotated traces and self-play curricula \citep{zhang2026a,acikgoz2026}. Other approaches combine teacher-generated trial-and-correction imitation with interactive RL or refine training data from errors \citep{gao2026,zhang2026b}. These studies mainly improve tool use through demonstration data, reward design, and interactive exploration. We focus on cases where SLMs struggle to learn effectively through autonomous exploration, using teacher guidance to supplement learning signals from their own attempts.

\paragraph{Teacher-Guided Learning.}
Teacher-guided learning uses complete demonstrations or student-adapted targets and trajectories \citep{kang2025,shing2025,tang2026}. On-policy distillation supervises student-generated outputs to reduce the mismatch between teacher references and student histories \citep{gu2024,agarwal2024}. Student-history corrections and divergence-aware feedback \citep{lyu2025,guo2026,zhong2026}, together with staged, mixed, or interleaved assistance \citep{ni2026,yan2025,ma2026}, further adapt guidance to students' learning processes. Beyond the content of guidance, when teachers intervene also matters for knowledge transfer. Reward- and failure-based methods selectively distill responses, supply teacher-derived rewards, or route trajectories between reward optimization and feedback-conditioned self-distillation \citep{xu2025a,kotoge2026,li2026a}. At the group level, DyME and RSTG trigger demonstration learning or adaptive guidance when all sampled responses fail \citep{liu2026,han2026}. These studies promote knowledge transfer by adapting supervision and guidance allocation to students' learning. \ours selectively provides teacher guidance based on student task performance and combines online feedback with demonstration supervision to support knowledge transfer and autonomous exploration.

\section{Preliminaries}
\label{sec:preliminaries}

\subsection{Tool-Use Tasks and Feedback}

Given a request and tool descriptions $x$, an SLM with policy $\pi_\theta$ generates a trajectory $y=(y_1,\ldots,y_T)$ of reasoning, structured tool calls, and responses. A tool call must identify an appropriate function and specify arguments consistent with its interface and the user's request. Each token depends on $x$ and its prefix $y_{<t}=(y_1,\ldots,y_{t-1})$.

A deterministic Composite Checker evaluates each completed trajectory $y_j$ using
\begin{equation}
R_j
=R_{\mathrm{format}}+R_{\mathrm{correct}}.
\label{eq:final_reward}
\end{equation}
The format component checks response structure, while correctness compares tool names, argument names, and argument values against reference calls. This scalar reward credits partially correct calls but provides no token-level correction targets.

\subsection{Group Relative Policy Optimization}

Group Relative Policy Optimization (GRPO)~\citep{shao2024} samples $G$ trajectories $\{y_j\}_{j=1}^{G}$ from the behavior policy $\pi_{\theta_{\mathrm{old}}}$ for each input $x$, and computes group-relative advantages from rewards $R_j$:
\begin{equation}
A_j
=\frac{R_j-\operatorname{mean}(\{R_1,\ldots,R_G\})}
{\operatorname{std}(\{R_1,\ldots,R_G\})+\epsilon_A},
\label{eq:grpo_advantage}
\end{equation}
where $\epsilon_A$ ensures numerical stability. Above-average rewards yield positive advantages, while below-average rewards yield negative ones. When a group's attempts receive the same low reward, their task advantages are zero, providing no task-driven update direction for improving tool selection or argument construction in subsequent attempts.

For a generated token $y_{j,t}$, the policy ratio is
\begin{equation}
\rho_{j,t}(\theta)
=\frac{\pi_\theta(y_{j,t}\mid x,y_{j,<t})}
{\pi_{\theta_{\mathrm{old}}}(y_{j,t}\mid x,y_{j,<t})}.
\label{eq:policy_ratio}
\end{equation}
With $\bar\rho_{j,t}=\operatorname{clip}(\rho_{j,t},1-\epsilon,1+\epsilon)$, the clipped GRPO loss is
\begin{equation}
\begin{aligned}
\mathcal{L}_{\mathrm{GRPO}}(\theta)
={}&-\mathbb{E}_{j,t}\bigl[\min\bigl(
\\[-2pt]
&\rho_{j,t}(\theta)A_j,\bar\rho_{j,t}(\theta)A_j\bigr)\bigr].
\end{aligned}
\label{eq:grpo_objective}
\end{equation}
Here, $\epsilon$ is the clipping parameter, and the expectation averages over valid student-generated token positions. Tokens share their trajectory's task advantage. This provides a trajectory-level learning signal without explicit correction targets for individual tool-selection or argument-construction decisions within the sampled trajectory.

\section{Method}
\label{sec:method}

\subsection{Overview}

Selective Teacher-guided Policy Optimization (\ours) helps SLMs acquire tool-use capabilities that are difficult to learn through autonomous exploration. As shown in Figure~\ref{fig:framework}, it combines online teacher guidance on low-reward student attempts with complete teacher demonstrations throughout RL. Online guidance addresses decisions made during exploration, while demonstration supervision provides complete reference behaviors that may be absent from student rollouts.

\begin{figure*}[!t]
\centering
\includegraphics[width=\textwidth]{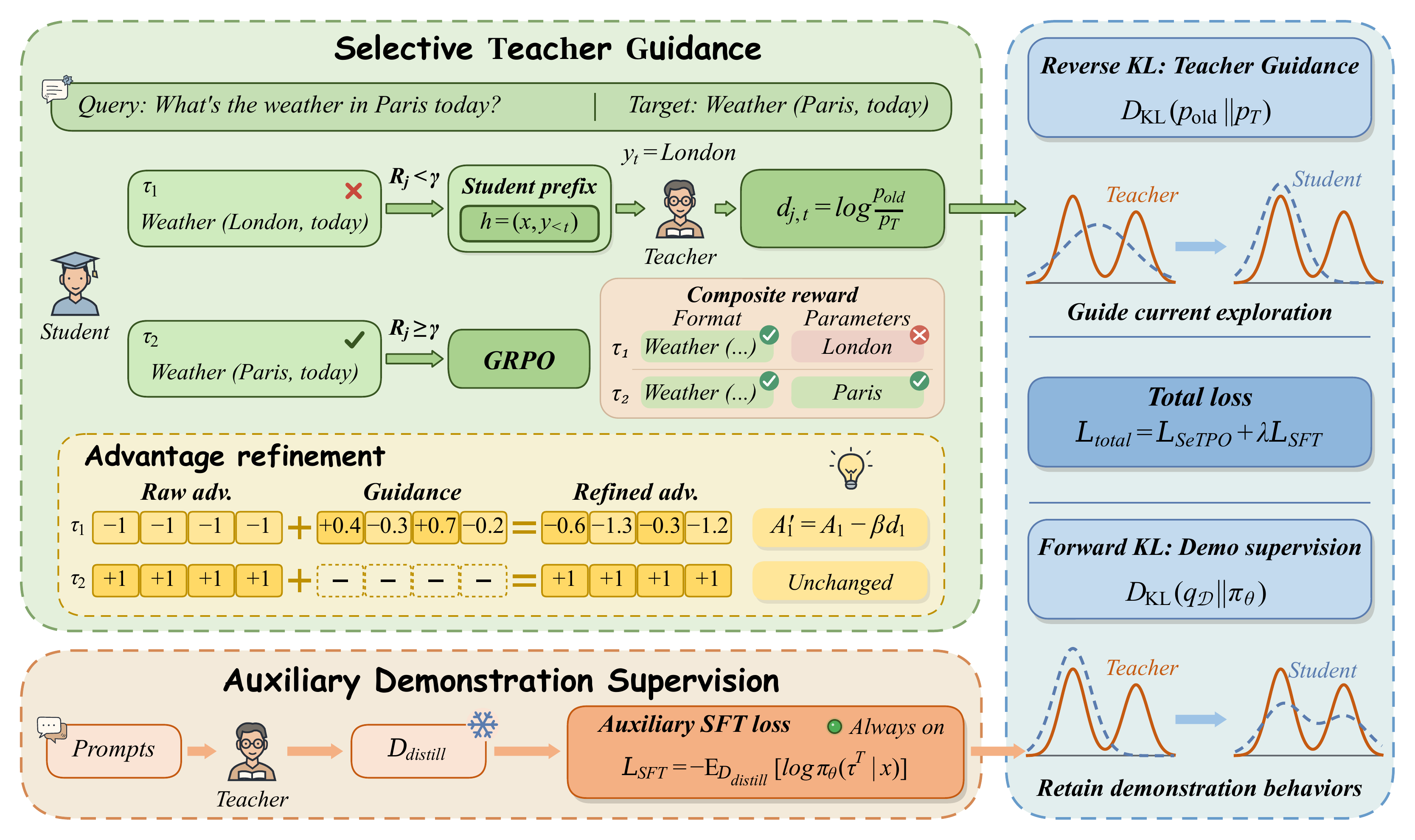}
\caption{Overview of \ours. Selective online teacher guidance and auxiliary
supervision from complete teacher trajectories jointly support policy
optimization. A composite checker selects low-reward
student trajectories for teacher scoring; token-wise log-probability
discrepancies refine their GRPO advantages. Both probabilities are evaluated for the same sampled token under the student's prefix.}
\label{fig:framework}
\end{figure*}

\subsection{Selective Teacher-guided Policy Optimization}
\label{sec:setpo}

\paragraph{Selective Teacher Guidance.}
RL allows students to improve through their own tool-use attempts, but unsuccessful exploration may provide limited information about how to revise individual decisions. We therefore focus online teacher guidance on low-reward attempts that require additional learning signals, while preserving task-driven online updates for higher-reward attempts whose behavior is already supported by task outcomes. For a student trajectory $y_j$ with reward $R_j$, we define the intervention indicator as
\begin{equation}
\mathbb{I}_j=\mathbb{I}[R_j<\gamma],
\label{eq:intervention_indicator}
\end{equation}
where $\gamma$ is the intervention threshold. The gate is evaluated separately for each trajectory, so attempts within the same rollout group can receive different guidance according to their individual outcomes. The same prompt can therefore receive different guidance as the student's sampled attempts change. For selected trajectories, the frozen teacher evaluates generated tokens under the student's own prefix $h_{j,t}=(x,y_{j,<t})$. Teacher feedback therefore addresses the student's actual decision context while the student generates its own attempts.

\paragraph{Teacher-Guided Policy Updates.}
At the student-generated prefix $h_{j,t}$, let $p_{\mathrm{old}}(v)=\pi_{\theta_{\mathrm{old}}}(v\mid h_{j,t})$ and $p_T(v)=\pi_T(v\mid h_{j,t})$ denote the behavior-policy and teacher distributions. We compute the token-wise log-probability discrepancy:
\begin{equation}
d_{j,t}
=\log p_{\mathrm{old}}(y_{j,t})-\log p_T(y_{j,t}).
\label{eq:token_discrepancy}
\end{equation}
At a fixed prefix, its expectation under the behavior policy equals the reverse KL $D_{\mathrm{KL}}(p_{\mathrm{old}}\|p_T)$, making $d_{j,t}$ a single-sample estimate~\citep{gu2024}. We incorporate this feedback into the task advantage through the reward gate:
\begin{equation}
A'_{j,t}
=A_j-\mathbb{I}_j\beta d_{j,t},
\label{eq:refined_advantage}
\end{equation}
where $\beta$ controls guidance strength. Tokens assigned higher probability by the teacher than by the behavior policy receive an increased advantage, and conversely for lower teacher probability. The refined advantage yields the clipped loss
\begin{equation}
\begin{aligned}
\mathcal{L}_{\text{\ours}}(\theta)
={}&-\mathbb{E}_{j,t}\bigl[\min\bigl(
\\[-2pt]
&\rho_{j,t}(\theta)A'_{j,t},
\bar\rho_{j,t}(\theta)A'_{j,t}\bigr)\bigr].
\end{aligned}
\label{eq:setpo_objective}
\end{equation}

\noindent The task advantage provides a shared trajectory-level signal, whereas teacher guidance differentiates tokens within selected attempts. Both signals act on the same student-generated trajectories through a single clipped policy objective.

\begin{table*}[!t]
  \centering
  \newsavebox{\mainresultstablebox}
\footnotesize
\setlength{\tabcolsep}{2.6pt}
\renewcommand{\arraystretch}{1.08}
\begin{lrbox}{\mainresultstablebox}%
\begin{tabular}{l cccccc >{\hspace{8pt}}c ccc}
\toprule
\multirow{2}{*}{\textbf{Model / Method}}
& \multicolumn{6}{c}{\textbf{BFCL-v4 (Acc)}}
& \multicolumn{4}{c}{\textbf{API-Bank (Acc)}} \\
\cmidrule(lr){2-7}\cmidrule(l{12pt}){8-11}
& \textbf{NL-P} & \textbf{NL-PM} & \textbf{Live-M} & \textbf{MT-B} & \textbf{MT-MF} & \textbf{Over.}
& \textbf{L1} & \textbf{L2} & \textbf{L3} & \textbf{Over.} \\
\midrule
\rowcolor{teacherbluegray}
Qwen2.5-1.5B-Instruct & 82.50 & 79.00 & 64.39 & 6.50 & 4.50 & 24.13 & 69.67 & 61.19 & 39.69 & 62.14 \\
\midrule
\textbf{Qwen2.5-0.5B-Instruct} & 0.00 & 0.00 & 0.57 & 0.00 & 0.00 & 10.38 & 22.81 & 8.96 & 22.14 & 21.11 \\
\quad + SFT & 50.00 & 26.50 & 33.24 & 1.50 & 1.00 & 14.68 & 36.84 & 17.91 & 27.48 & 32.66 \\
\quad + GRPO & 56.00 & 35.50 & 33.62 & 1.00 & 0.00 & 10.67 & 50.38 & \textbf{41.79} & 33.59 & 45.73 \\
\quad + SFT+GRPO & 61.00 & 41.00 & 42.74 & 1.00 & 0.00 & 11.29 & 38.10 & 17.91 & 32.06 & 34.51 \\
\quad + KD & 59.50 & 49.00 & \textbf{44.82} & 0.50 & 0.00 & 16.45 & 41.85 & 25.37 & 17.56 & 34.67 \\
\quad + OPD & 54.50 & 47.50 & 18.42 & 1.00 & \textbf{1.50} & 16.17 & 52.13 & 22.39 & \textbf{35.88} & 45.23 \\
\rowcolor{oursorange}
\quad + \ours & \textbf{64.00} & \textbf{55.50} & 34.38 & \textbf{3.00} & 0.50 & \textbf{18.99} & \textbf{52.63} & 28.36 & 35.11 & \textbf{46.06} \\
\addlinespace[3pt]
\midrule
\rowcolor{teacherbluegray}
Qwen2.5-7B-Instruct & 91.50 & 89.00 & 75.97 & 22.50 & 15.50 & 30.69 & 71.68 & 59.70 & 35.11 & 62.31 \\
\midrule
\textbf{Qwen2.5-1.5B-Instruct} & 4.50 & 25.00 & 3.99 & 0.00 & 0.00 & 11.14 & 56.39 & 53.73 & 38.17 & 52.09 \\
\quad + SFT & 59.00 & 36.00 & 50.71 & 2.00 & 0.00 & 17.51 & 61.90 & 47.76 & \textbf{51.91} & 58.12 \\
\quad + GRPO & 81.00 & 78.50 & \textbf{65.05} & 6.00 & 1.00 & 19.43 & 59.65 & 52.24 & 38.93 & 54.27 \\
\quad + SFT+GRPO & 80.50 & 71.00 & 61.63 & 2.00 & 2.00 & 22.05 & 64.91 & 58.21 & 48.85 & 60.64 \\
\quad + KD & 74.50 & 65.50 & 53.75 & 5.00 & 2.00 & 21.60 & 61.90 & 55.22 & 35.11 & 55.28 \\
\quad + OPD & 78.50 & 61.50 & 52.14 & 5.50 & 3.00 & 22.17 & 65.16 & 58.21 & 37.40 & 58.29 \\
\rowcolor{oursorange}
\quad + \ours & \textbf{82.50} & \textbf{79.00} & 64.39 & \textbf{6.50} & \textbf{4.50} & \textbf{24.13} & \textbf{69.67} & \textbf{61.19}\teacherbetter & 39.69\teacherbetter & \textbf{62.14} \\
\addlinespace[3pt]
\midrule
\rowcolor{teacherbluegray}
Qwen2.5-14B-Instruct & 94.50 & 87.50 & 73.79 & 33.00 & 27.00 & 32.90 & 70.68 & 52.24 & 42.75 & 62.48 \\
\midrule
\textbf{Qwen2.5-3B-Instruct} & 14.50 & 24.50 & 61.54 & 5.50 & 1.50 & 20.35 & 42.86 & 22.39 & 18.32 & 35.18 \\
\quad + SFT & 62.00 & 55.50 & 45.30 & 6.00 & 1.00 & 14.26 & 59.90 & 41.79 & \textbf{57.25} & 57.29 \\
\quad + GRPO & 83.50 & 78.00 & 69.99 & \textbf{10.50} & 3.50 & 24.81 & 67.42 & 58.21 & 35.11 & 59.30 \\
\quad + SFT+GRPO & 87.50 & \textbf{85.50} & 69.99 & 9.00 & 4.50 & 25.37 & 69.92 & 58.21 & 37.40 & 61.47 \\
\quad + KD & 84.50 & 80.50 & 64.20 & 7.50 & 4.00 & 24.53 & 67.92 & 62.69 & 31.30 & 59.30 \\
\quad + OPD & 83.50 & 80.50 & 67.24 & 7.00 & \textbf{5.50} & 24.76 & 69.67 & 62.69 & 38.93 & 62.14 \\
\rowcolor{oursorange}
\quad + \ours & \textbf{88.50} & 84.50 & \textbf{71.98} & 6.50 & \textbf{5.50} & \textbf{25.39} & \textbf{72.43}\teacherbetter & \textbf{67.16}\teacherbetter & 35.88 & \textbf{63.82}\teacherbetter \\
\bottomrule
\end{tabular}%
\end{lrbox}%
\begin{minipage}{\wd\mainresultstablebox}
\centering
\usebox{\mainresultstablebox}
\caption{Main experimental results on BFCL-v4 and API-Bank across 0.5B, 1.5B, and 3B students. Each group starts with its corresponding teacher (\tablecolorlegend{teacherbluegray}~blue-gray row), followed by the initial student and its trained variants; \tablecolorlegend{oursorange}~pale orange rows highlight \ours. The best student result within each group is shown in \textbf{bold}. \teacherbetter~Green arrows mark \ours scores above the corresponding teacher.}
\label{table:main_experiment}
\end{minipage}
\end{table*}

\subsection{Auxiliary Demonstration Supervision}
\label{sec:auxiliary_demonstration}

Online guidance evaluates decisions that the student has already generated, but does not itself supply complete reference trajectories. Within a limited exploration budget, student rollouts may still lack effective sequences of tool selection and argument construction, or fail to maintain the required response structure. We therefore retain supervision from complete teacher demonstrations throughout RL. Before policy optimization, the frozen teacher generates a reference trajectory for each training prompt, forming $\mathcal{D}_{\mathrm{distill}}=\{(x_i,\tau_i^T)\}_{i=1}^{N_{\mathrm{RL}}}$. These references provide fixed supervision across policy updates, while student-generated trajectories evolve with the policy. We minimize the following demonstration loss:
\begin{equation}
\begin{aligned}
\mathcal{L}_{\mathrm{SFT}}(\theta)
={}&-\mathbb{E}_{(x,\tau^T)\sim\mathcal{D}_{\mathrm{distill}}}
\\[-2pt]
&\left[\sum_t\log\pi_\theta
(\tau_t^T\mid x,\tau_{<t}^T)\right].
\end{aligned}
\label{eq:auxiliary_sft}
\end{equation}
This supervision uses teacher-generated prefixes and remains active independently of the intervention gate. Let $q_{\mathcal D}(\tau\mid x)$ denote the empirical demonstration distribution. Minimizing this cross-entropy is equivalent to minimizing the forward KL $D_{\mathrm{KL}}\!\left(q_{\mathcal D}(\cdot\mid x)\|\pi_\theta(\cdot\mid x)\right)$, averaged over training prompts. The demonstration distribution is fixed, so its entropy is independent of student parameters.

\subsection{Theoretical Analysis}
\label{sec:theoretical_analysis}

Combining the online policy objective and auxiliary demonstration supervision yields
\begin{equation}
\mathcal{L}_{\mathrm{total}}(\theta)
=\mathcal{L}_{\text{\ours}}(\theta)
+\lambda\mathcal{L}_{\mathrm{SFT}}(\theta),
\label{eq:total_loss}
\end{equation}
where $\lambda$ controls the strength of demonstration supervision. The online objective combines task-reward optimization with reverse-KL-derived guidance on student-generated attempts, while auxiliary SFT minimizes forward KL from the empirical demonstration distribution $q_{\mathcal D}$ to the student policy. To examine how these complementary forms of supervision support policy improvement while preserving teacher-supported behaviors, we consider the following idealized objective.

\begin{assumption}[Bidirectional KL-Regularized Objective]
\label{ass:population}
For a fixed prompt, let $q$ be a teacher reference distribution with full support over a finite response space $\mathcal Y$, and let $R(y)$ denote the bounded task reward. We consider the population objective
\begin{equation}
\begin{aligned}
\mathcal F(p)
={}&\mathbb E_{y\sim p}[R(y)]\\
&-aD_{\mathrm{KL}}(p\|q)-bD_{\mathrm{KL}}(q\|p),
\end{aligned}
\label{eq:population_objective}
\end{equation}
where $a,b>0$ represent the effective strengths of reverse- and forward-KL regularization, respectively. This idealization represents the online teacher and demonstration distributions by a common reference $q$, and abstracts away selective intervention, group-relative advantage normalization, clipping, and finite-sample effects. The analysis optimizes over all response distributions rather than a restricted student parameterization.
\end{assumption}

\begin{theorem}[Reward Improvement and Teacher Behavior Preservation]
\label{thm:teacher_preservation}
Under Assumption~\ref{ass:population}, $\mathcal F$ admits a unique maximizer $p^\star$ over the probability simplex $\Delta(\mathcal Y)$. Let
\begin{equation}
\Delta_R=\max_{y\in\mathcal Y}R(y)-\min_{y\in\mathcal Y}R(y).
\label{eq:reward_range}
\end{equation}
Then the optimal policy achieves an expected reward no lower than that of the teacher:
\begin{equation}
\mathbb E_{p^\star}[R]\geq\mathbb E_q[R].
\label{eq:reward_improvement}
\end{equation}
Moreover, for every $y\in\mathcal Y$,
\begin{equation}
\frac{b}{b+\Delta_R}\leq\frac{p^\star(y)}{q(y)}
\leq\exp\left(\frac{\Delta_R}{a}\right).
\label{eq:probability_ratio_bound}
\end{equation}
Consequently, for any response set $\mathcal E\subseteq\mathcal Y$,
\begin{equation}
p^\star(\mathcal E)\geq\frac{b}{b+\Delta_R}\,q(\mathcal E).
\label{eq:behavior_preservation}
\end{equation}
\end{theorem}

Under this idealized objective, Theorem~\ref{thm:teacher_preservation} establishes that bidirectional KL regularization permits expected reward improvement while bounding deviation from the teacher distribution. Specifically, forward KL provides a lower bound on the probability retained for teacher-supported behaviors, whereas reverse KL limits their relative amplification. These complementary constraints support task improvement while preserving coverage of teacher behaviors, providing a theoretical rationale for combining online teacher guidance with persistent demonstration supervision. The proof is provided in Appendix~\ref{app:teacher_preservation}.


\section{Experiments}
\label{sec:experiments}

\subsection{Experimental Setup}

\paragraph{Datasets and Models.}
We use the ToolRL datasets~\citep{qian2025}, comprising 400 expert examples for cold-start SFT and 4,000 examples from ToolACE~\citep{liu2025}, Hammer~\citep{lin2025}, and xLAM~\citep{zhang2025} for subsequent training. Qwen2.5-Instruct~\citep{yang2024} students of 0.5B, 1.5B, and 3B are paired with frozen 1.5B, 7B, and 14B teachers, respectively. We evaluate tool use on BFCL-v4~\citep{patil2025} and API-Bank~\citep{li2023}, and general reasoning on GSM8K~\citep{cobbe2021} and HotpotQA~\citep{yang2018}.

\begingroup
\brokenpenalty=0
\paragraph{Baselines.}
We consider six representative baselines: (1) \textbf{Vanilla:} the original instruction-tuned model; (2) \textbf{SFT:} supervised learning from provided responses and tool-use trajectories; (3) \textbf{GRPO:} RL starting from Vanilla; (4) \textbf{SFT+GRPO:} GRPO following expert cold-start initialization; (5) \textbf{offline KD:} supervision from teacher-generated reference trajectories; and (6) \textbf{OPD:} online distillation from frozen-teacher feedback on student-generated trajectories. Neither GRPO nor SFT+GRPO uses online teacher guidance or auxiliary SFT during policy optimization, allowing comparison with reward-based learning under different initializations. Appendix~\ref{sec:implementation_details} provides data splits, training settings, and details of the baseline training objectives.

\par
\endgroup

\begin{table}[!t]
  \centering
  \small
  \renewcommand{\arraystretch}{1.1}
  \setlength{\tabcolsep}{4pt}

  \begin{tabular*}{0.95\columnwidth}{@{\extracolsep{\fill}}lcc@{}}
    \toprule
    \textbf{Method}
    & \textbf{BFCL-v4}
    & \textbf{API-Bank} \\
    \midrule

    Always-on guidance
    & 16.51
    & 39.53 \\

    Random intervention
    & 13.29
    & 42.21 \\

    \ours
    & \textbf{18.99}
    & \textbf{46.06} \\

    \bottomrule
  \end{tabular*}

  \caption{
  Comparison of online teacher-intervention strategies for Qwen2.5-0.5B
  on BFCL-v4 and API-Bank. The variants retain the same formulation of teacher
  feedback and other training components, while applying guidance to all
  trajectories, a random subset, or low-reward trajectories. The best result
  on each benchmark is shown in bold.
}
  \label{tab:teacher_intervention_strategy}
\end{table}

\begin{figure}[b]
    \centering
    \includegraphics[width=\columnwidth]{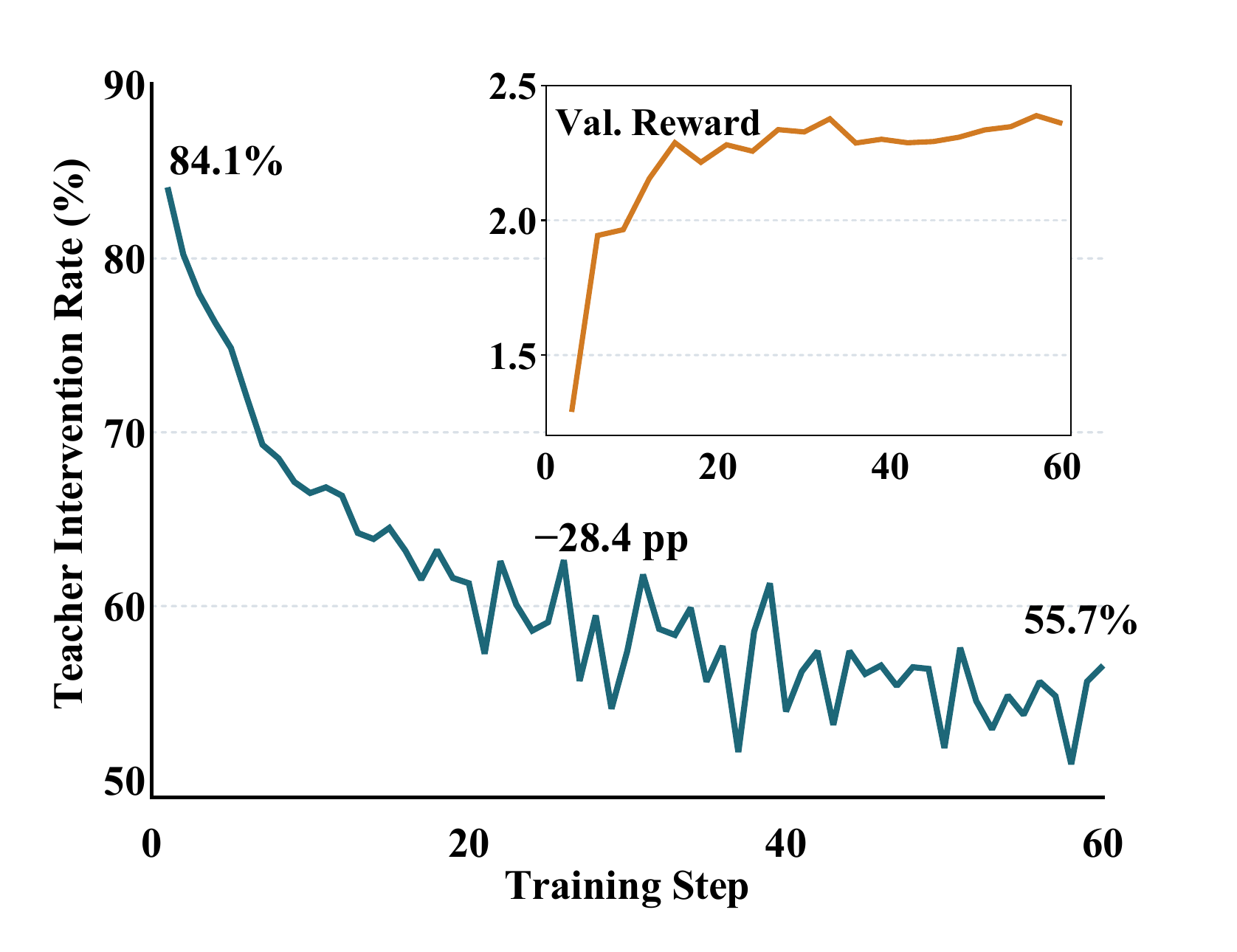}
    \caption{
    Teacher-intervention rate during \ours training.
    The inset shows the validation reward.
    }
    \label{fig:teacher_intervention_dynamics}
\end{figure}

\subsection{Main Results}
\label{sec:main_results}

Table~\ref{table:main_experiment} compares tool-use performance at three student scales on BFCL-v4 and API-Bank. \ours achieves the highest overall student scores at every scale on both benchmarks. For the 0.5B student, it improves BFCL-v4 accuracy by 8.32 percentage points over GRPO and 2.82 points over pure OPD. These gains may reflect the complementary roles of student exploration and teacher supervision. Student-generated attempts extend learning beyond fixed SFT demonstrations, while teacher guidance supplies information that limited GRPO exploration can miss. The comparison with pure OPD further suggests the value of task-based optimization in turning teacher knowledge into effective tool-use behavior. The 3B student also exceeds its 14B teacher on API-Bank, achieving 63.82\% versus 62.48\%. Gains vary by scale and category; margins over the strongest baselines are small for 0.5B on API-Bank and 3B on BFCL-v4.

\begin{figure}[tbp]
    \centering
    \includegraphics[width=\linewidth]{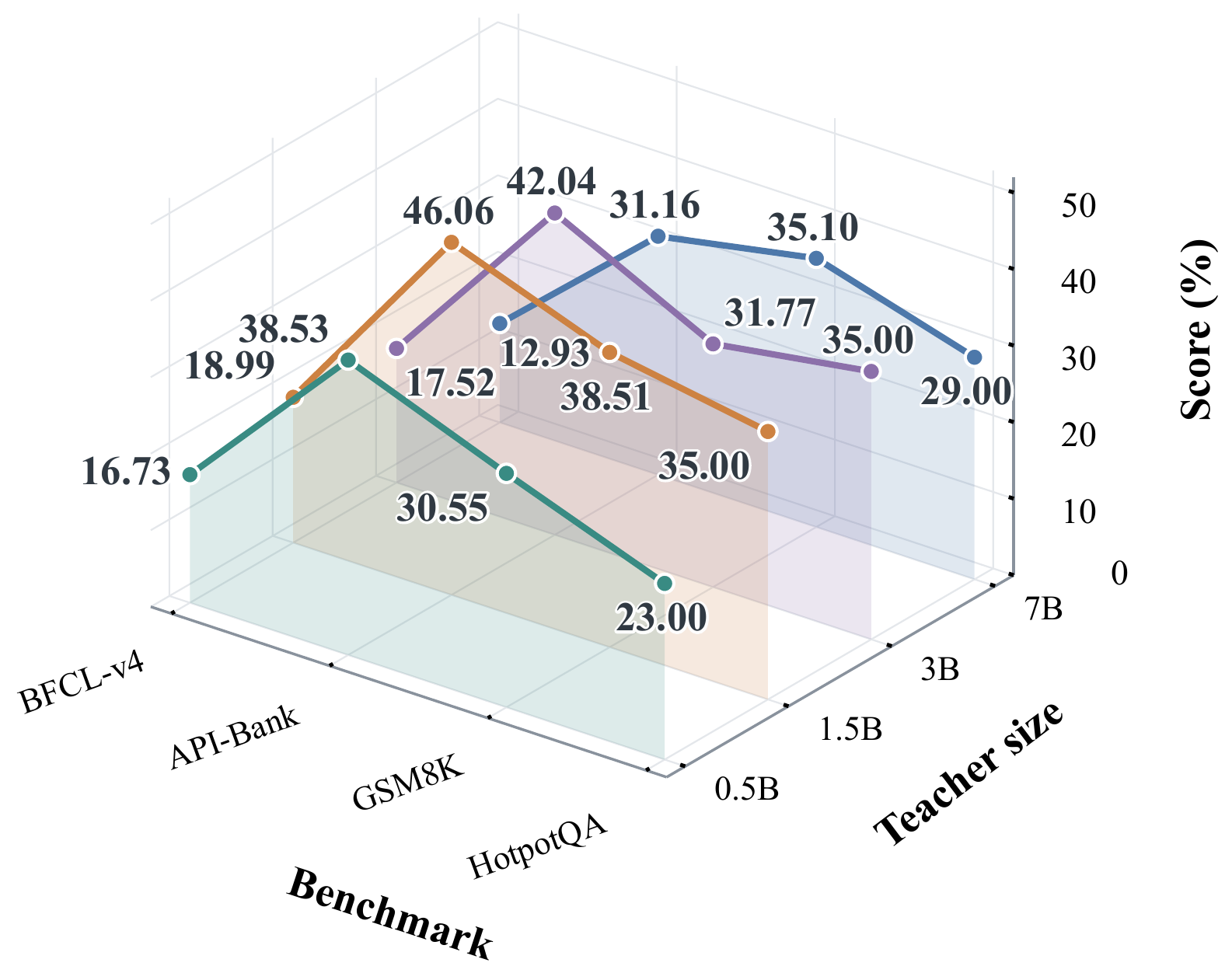}
    \caption{
    Performance of a Qwen2.5-0.5B student guided by teachers
    of different sizes across four benchmarks.
    }
    \label{fig:ablation_teacher}
\end{figure}

\begin{figure}[b]
    \centering
    \includegraphics[width=\columnwidth]{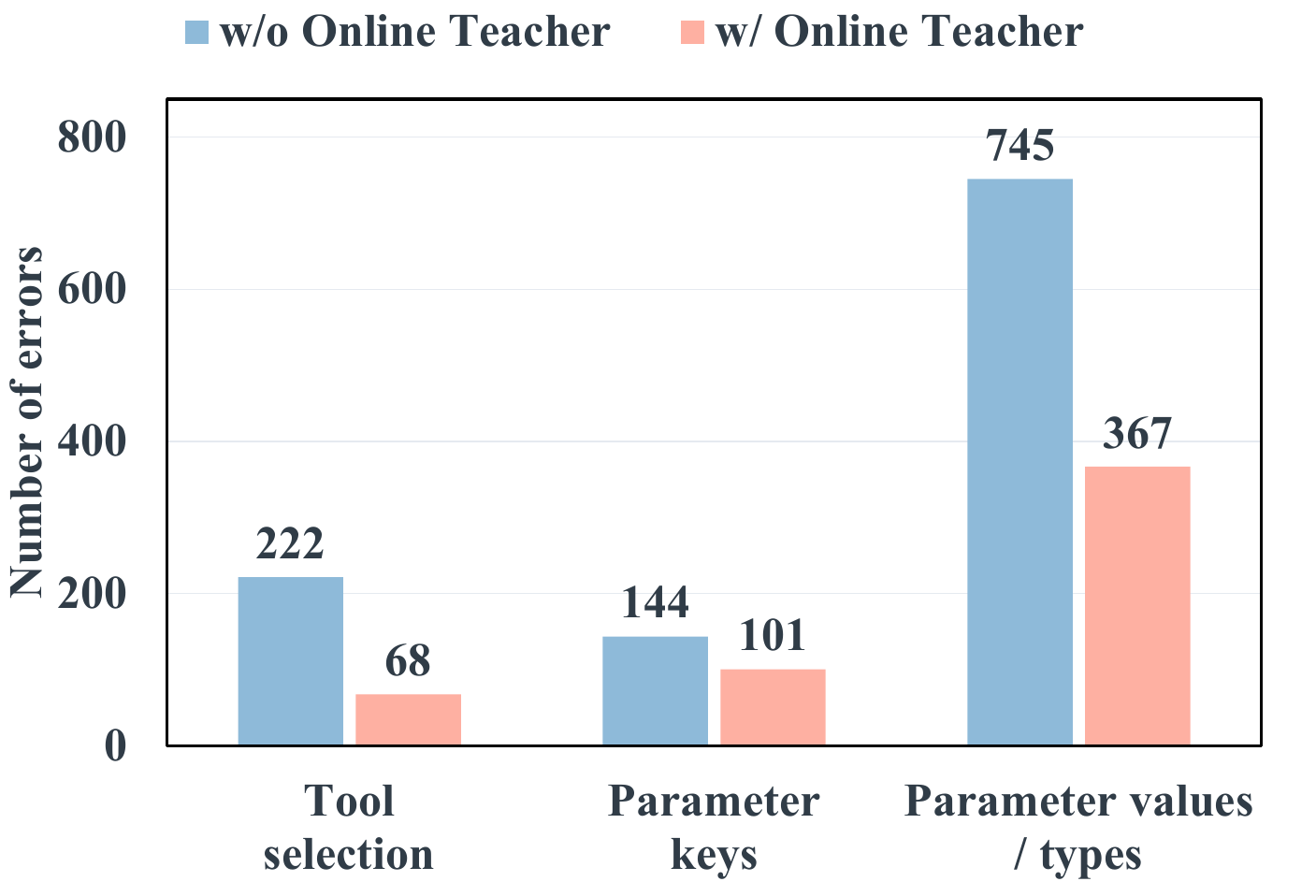}
    \caption{
        Counts of tool-selection and argument-construction errors for Qwen2.5-0.5B
        on BFCL-v4. Blue and coral denote the comparison model without online
        teacher guidance and \ours, respectively.
    }
    \label{fig:online_teacher_error_counts}
\end{figure}

\begin{table}[tbp]
  \centering
  \newsavebox{\reasoningtablebox}
\footnotesize
\setlength{\tabcolsep}{3.5pt}
\renewcommand{\arraystretch}{1.08}
\begin{lrbox}{\reasoningtablebox}%
\begin{tabular}{l c >{\hspace{8pt}}c}
\toprule
\textbf{Model / Method} & \textbf{GSM8K} & \textbf{HotpotQA} \\
\midrule
\rowcolor{teacherbluegray}
Qwen2.5-1.5B-Instruct & 60.88 & 46.0 \\
\midrule
\textbf{Qwen2.5-0.5B-Instruct} & 31.99 & 3.0 \\
\quad + SFT & 34.87 & 23.0 \\
\quad + GRPO & 35.03 & 10.0 \\
\quad + SFT+GRPO & 38.21 & 18.0 \\
\quad + KD & 24.49 & 18.0 \\
\quad + OPD & 34.27 & 14.0 \\
\rowcolor{oursorange}
\quad + \ours & \textbf{38.51} & \textbf{35.0} \\
\addlinespace[3pt]
\midrule
\rowcolor{teacherbluegray}
Qwen2.5-7B-Instruct & 71.42 & 63.0 \\
\midrule
\textbf{Qwen2.5-1.5B-Instruct} & 56.71 & 4.0 \\
\quad + SFT & 53.53 & 34.0 \\
\quad + GRPO & 55.95 & 34.0 \\
\quad + SFT+GRPO & 57.01 & 26.0 \\
\quad + KD & 47.16 & 44.0 \\
\quad + OPD & 48.07 & 40.0 \\
\rowcolor{oursorange}
\quad + \ours & \textbf{60.88} & \textbf{46.0} \\
\addlinespace[3pt]
\midrule
\rowcolor{teacherbluegray}
Qwen2.5-14B-Instruct & 61.26 & 72.0 \\
\midrule
\textbf{Qwen2.5-3B-Instruct} & 45.79 & 46.0 \\
\quad + SFT & 44.28 & 17.0 \\
\quad + GRPO & 47.08 & 44.0 \\
\quad + SFT+GRPO & 45.56 & 32.0 \\
\quad + KD & 54.74 & 39.0 \\
\quad + OPD & 52.77 & \textbf{49.0} \\
\rowcolor{oursorange}
\quad + \ours & \textbf{58.91} & 39.0 \\
\bottomrule
\end{tabular}%
\end{lrbox}%
\begin{minipage}{\wd\reasoningtablebox}
\centering
\usebox{\reasoningtablebox}
\caption{General reasoning results on GSM8K and HotpotQA. \tablecolorlegend{teacherbluegray}~Blue-gray rows denote each group's teacher; \tablecolorlegend{oursorange}~pale orange rows highlight \ours. \textbf{Bold} marks the best student score within each group.}
\label{table:reasoning_results}
\end{minipage}
\end{table}

\subsection{Further Analysis}
\label{sec:further_analysis}

\paragraph{\normalfont\itshape Allocation of Teacher Guidance.}
\label{sec:teacher_allocation}
\textbf{How should teacher guidance be allocated during training?}
Table~\ref{tab:teacher_intervention_strategy} compares always-on, budget-matched random, and low-reward teacher guidance for the 0.5B student. \ours performs best on BFCL-v4 and API-Bank with the feedback formulation and other training components held fixed. Always-on guidance may constrain high-reward attempts, while random allocation may overlook unsuccessful attempts that need additional learning signals. Selective guidance concentrates feedback on low-reward attempts while retaining task-based online updates elsewhere. Figure~\ref{fig:teacher_intervention_dynamics} shows the intervention rate and validation reward during training. Intervention decreases from 84.1\% to 55.7\% while validation reward generally rises. With a fixed threshold, this pattern suggests that guidance adapts to the student's improving performance while continuing to target low-reward attempts.

\paragraph{\normalfont\itshape Impact of Teacher Size.}
\label{sec:teacher_scale_details}
\textbf{How does teacher size affect student performance?}
Figure~\ref{fig:ablation_teacher} compares a Qwen2.5-0.5B student guided by teachers
of different sizes across four benchmarks, with detailed scores in
Appendix~\ref{sec:teacher_selection_results}. The 1.5B teacher yields the best results
on BFCL-v4, API-Bank, and GSM8K, and ties with the 3B teacher on HotpotQA. Student
performance does not improve monotonically across these teacher sizes. The 1.5B teacher's advantage may reflect a balance between sufficient capability and a manageable teacher--student capacity gap, facilitating knowledge transfer.

\paragraph{\normalfont\itshape General Reasoning Performance.}
\label{sec:general_reasoning_results}
\textbf{How does the trained model perform on general reasoning tasks?}
Table~\ref{table:reasoning_results} compares GSM8K and HotpotQA performance after
tool-use training. \ours achieves the highest GSM8K scores among the compared student
methods at all three scales, and leads on HotpotQA at 0.5B and 1.5B. At these two
smaller scales, \ours also exceeds pure OPD on both tasks. For the 1.5B student, GSM8K accuracy increases from 48.07\% with OPD to 60.88\% with \ours. The gains across arithmetic and multi-hop question answering suggest that the benefits of tool-use training can extend beyond adherence to the required response structure.

\begingroup
\setlength{\intextsep}{2pt}
\begin{table}[!t]
  \centering
  \small
  \renewcommand{\arraystretch}{1.0}
  \setlength{\abovecaptionskip}{2pt}
  \setlength{\tabcolsep}{3pt}

  \begin{tabular*}{0.95\columnwidth}{@{\extracolsep{\fill}}lcc@{}}
    \toprule
    \textbf{Method}
    & \textbf{BFCL-v4}
    & \textbf{API-Bank} \\
    \midrule

    \ours
    & \textbf{18.99}
    & \textbf{46.06} \\

    Trajectory-level signal
    & 7.79\,\ablationchange{-11.20}
    & 32.50\,\ablationchange{-13.56} \\

    w/o SFT loss
    & 16.05\,\ablationchange{-2.94}
    & 37.35\,\ablationchange{-8.71} \\

    w/o Online Teacher
    & 10.78\,\ablationchange{-8.21}
    & 39.36\,\ablationchange{-6.70} \\

    GRPO+OPD
    & 15.85\,\ablationchange{-3.14}
    & 45.23\,\ablationchange{-0.83} \\

    \bottomrule
  \end{tabular*}

  \caption{
    Ablation results on BFCL-v4 and API-Bank.
    Values in parentheses indicate the percentage-point
    change relative to \ours.
  }
  \label{table:ablation_mechanism}
\end{table}
\endgroup

\begin{figure}[t]
    \centering
    \includegraphics[width=\columnwidth]{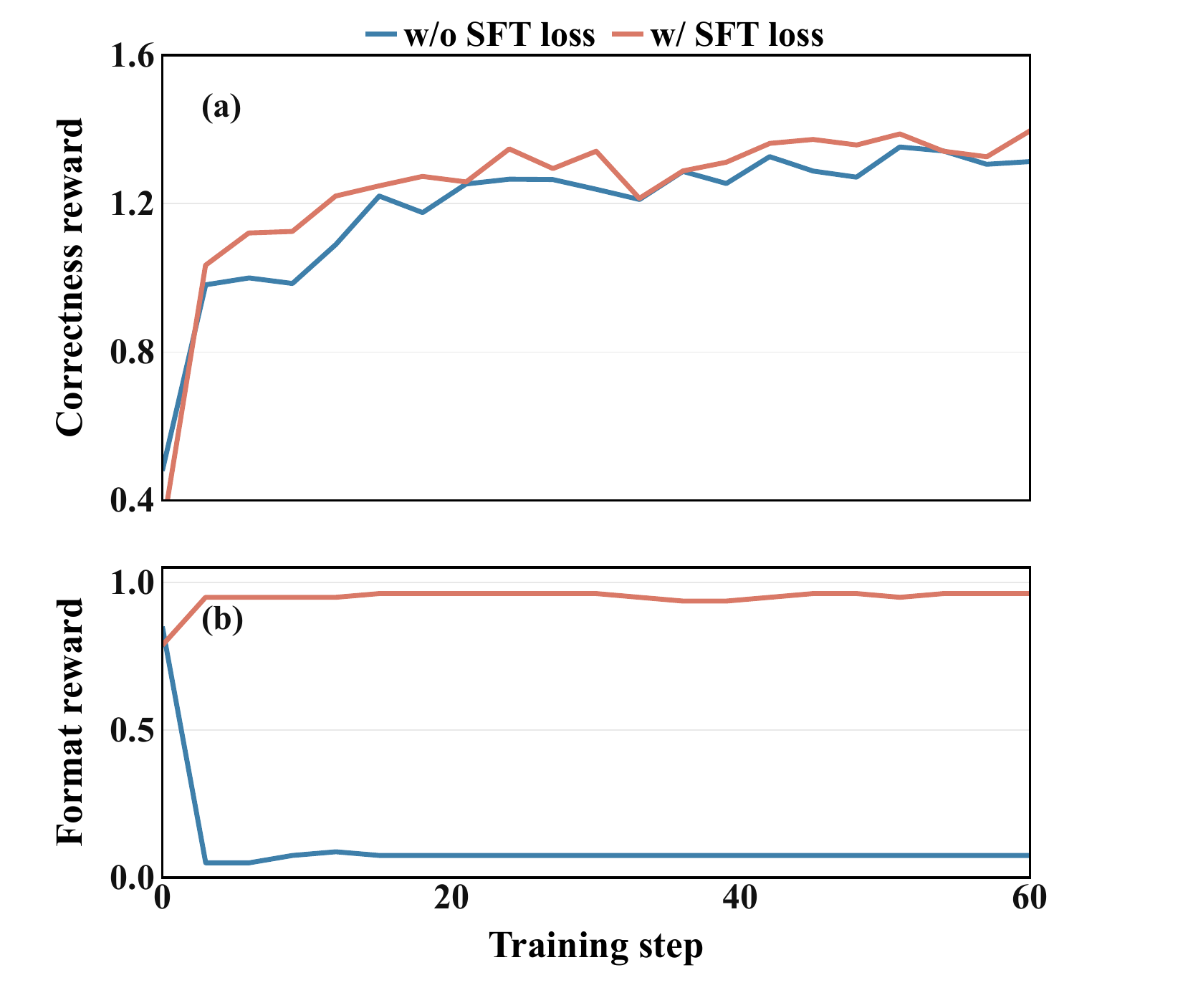}
    \caption{
        Training dynamics of \ours with and without the auxiliary SFT loss:
        (a) correctness reward and (b) format reward.
    }
    \label{fig:auxiliary_sft_supervision}
\end{figure}

\subsection{Ablation Study}
\label{sec:ablation_study}

\paragraph{Component Ablations}
\label{sec:component_ablations}
Table~\ref{table:ablation_mechanism} examines token-wise refinement, online teacher scoring, and auxiliary SFT. Replacing token-dependent feedback with a shared trajectory-level signal produces the largest decline, suggesting that teacher guidance should distinguish decisions within a tool-use attempt. Removing online teacher guidance eliminates reverse-KL feedback, while removing auxiliary SFT eliminates forward-KL demonstration supervision. Both changes reduce performance. GRPO+OPD, which adds a separate online distillation loss to GRPO, also scores lower on both benchmarks. Under the idealized objective in Theorem~\ref{thm:teacher_preservation}, reverse KL limits relative probability amplification, while forward KL provides a lower bound on teacher-behavior retention. \textbf{These results favor combining reverse-KL guidance for improving student attempts with forward-KL supervision for retaining demonstrated tool-use behaviors.}

\paragraph{Online Teacher Guidance}
\label{sec:online_teacher_error_analysis}
Figure~\ref{fig:online_teacher_error_counts} compares BFCL-v4 errors in tool selection, parameter keys, and parameter values/types for Qwen2.5-0.5B. \ours reduces all three error categories relative to the variant without online teacher guidance, with the largest reduction in parameter values/types. On the same test instances, \ours fully corrects 178 failures of the variant without online teacher guidance, including 157 involving argument construction. Reverse-KL feedback evaluates sampled tokens under student-generated prefixes, supplying token-level information beyond trajectory-level rewards. \textbf{These reductions suggest that reverse-KL guidance helps students learn more reliable tool-use behaviors from unsuccessful attempts.}

\paragraph{Auxiliary SFT Supervision}
\label{sec:auxiliary_sft_supervision}
Figure~\ref{fig:auxiliary_sft_supervision} compares training rewards with and without auxiliary SFT. With auxiliary SFT, correctness rewards are generally higher, while format rewards quickly rise and remain high. Without it, format rewards drop sharply and remain low despite improving correctness. In this setting, correctness gains alone do not ensure valid response structure. This pattern is consistent with the forward-KL preservation bound under the idealized objective. \textbf{Persistent forward-KL supervision helps preserve structured tool-use behavior throughout RL training.}

\section{Conclusion}

We proposed \ours to help SLMs learn tool-use behaviors that are difficult to discover through autonomous exploration. Selective reverse-KL teacher guidance improves students' own exploration. Persistent forward-KL demonstration supervision preserves teacher-supported behaviors. Under an idealized bidirectional-KL objective, we establish, at optimality, expected reward non-degradation relative to the teacher and a lower bound on teacher-behavior retention. Experiments on BFCL-v4 and API-Bank show consistent tool-use gains across three student scales, with fewer tool-selection and argument-construction errors and more stable response formats.

\section*{Limitations}

Our study focuses on Qwen2.5 students with 0.5B--3B parameters and two tool-use benchmarks. Applicability across other model families, tool domains, and long-horizon interactive tasks requires further evaluation. Teacher intervention depends on task rewards computed by a deterministic checker of response structure and reference-call correctness. Consequently, extending the framework to open-ended environments without reference calls requires reward designs that support both policy learning and guidance allocation. Online feedback requires access to a frozen teacher's token probabilities, while teacher scoring and reference generation add computational costs.

\bibliography{references}

@inproceedings{yu2025,
  author       = {Yuanqing Yu and
                  Zhefan Wang and
                  Weizhi Ma and
                  Shuai Wang and
                  Chuhan Wu and
                  Zhiqiang Guo and
                  Min Zhang},
  editor       = {Meeyoung Cha and
                  Chanyoung Park and
                  Noseong Park and
                  Carl Yang and
                  Senjuti Basu Roy and
                  Jessie Li and
                  Jaap Kamps and
                  Kijung Shin and
                  Bryan Hooi and
                  Lifang He},
  title        = {{StepTool}: Enhancing Multi-Step Tool Usage in {LLMs} via Step-Grained
                  Reinforcement Learning},
  booktitle    = {Proceedings of the 34th {ACM} International Conference on Information
                  and Knowledge Management, {CIKM} 2025, Seoul, Republic of Korea, November
                  10-14, 2025},
  pages        = {3952--3962},
  publisher    = {{ACM}},
  year         = {2025},
  url          = {https://doi.org/10.1145/3746252.3761391},
  doi          = {10.1145/3746252.3761391},
  bibsource    = {dblp computer science bibliography, https://dblp.org}
}

@article{lyu2025,
  author       = {Yuanjie Lyu and
                  Chengyu Wang and
                  Jun Huang and
                  Tong Xu},
  title        = {From Correction to Mastery: Reinforced Distillation of Large Language
                  Model Agents},
  journal      = {CoRR},
  volume       = {abs/2509.14257},
  year         = {2025},
  url          = {https://doi.org/10.48550/arXiv.2509.14257},
  doi          = {10.48550/ARXIV.2509.14257},
  eprinttype   = {arXiv},
  eprint       = {2509.14257},
  bibsource    = {dblp computer science bibliography, https://dblp.org}
}

@article{shao2024,
  author       = {Zhihong Shao and
                  Peiyi Wang and
                  Qihao Zhu and
                  Runxin Xu and
                  Junxiao Song and
                  Mingchuan Zhang and
                  Y. K. Li and
                  Y. Wu and
                  Daya Guo},
  title        = {{DeepSeekMath}: Pushing the Limits of Mathematical Reasoning in Open
                  Language Models},
  journal      = {CoRR},
  volume       = {abs/2402.03300},
  year         = {2024},
  url          = {https://doi.org/10.48550/arXiv.2402.03300},
  doi          = {10.48550/ARXIV.2402.03300},
  eprinttype    = {arXiv},
  eprint       = {2402.03300},
  bibsource    = {dblp computer science bibliography, https://dblp.org}
}

@inproceedings{bengio2015,
  author       = {Samy Bengio and
                  Oriol Vinyals and
                  Navdeep Jaitly and
                  Noam Shazeer},
  editor       = {Corinna Cortes and
                  Neil D. Lawrence and
                  Daniel D. Lee and
                  Masashi Sugiyama and
                  Roman Garnett},
  title        = {Scheduled Sampling for Sequence Prediction with Recurrent Neural Networks},
  booktitle    = {Advances in Neural Information Processing Systems 28: Annual Conference
                  on Neural Information Processing Systems 2015, December 7-12, 2015,
                  Montreal, Quebec, Canada},
  pages        = {1171--1179},
  year         = {2015},
  url          = {https://proceedings.neurips.cc/paper/2015/hash/e995f98d56967d946471af29d7bf99f1-Abstract.html},
  bibsource    = {dblp computer science bibliography, https://dblp.org}
}

@inproceedings{agarwal2024,
  author       = {Rishabh Agarwal and
                  Nino Vieillard and
                  Yongchao Zhou and
                  Piotr Stanczyk and
                  Sabela Ramos Garea and
                  Matthieu Geist and
                  Olivier Bachem},
  title        = {On-Policy Distillation of Language Models: Learning from Self-Generated
                  Mistakes},
  booktitle    = {The Twelfth International Conference on Learning Representations,
                  {ICLR} 2024, Vienna, Austria, May 7-11, 2024},
  publisher    = {OpenReview.net},
  year         = {2024},
  url          = {https://openreview.net/forum?id=3zKtaqxLhW},
  bibsource    = {dblp computer science bibliography, https://dblp.org}
}

@inproceedings{shenfeld2026,
  author       = {Idan Shenfeld and
                  Jyothish Pari and
                  Pulkit Agrawal},
  editor       = {Carl Vondrick and
                  Bharath Hariharan and
                  Colin Raffel and
                  Lerrel Pinto and
                  Diyi Yang and
                  Aleksandra Faust},
  title        = {{RL}'s {Razor}: Why Online Reinforcement Learning Forgets Less},
  booktitle    = {The Fourteenth International Conference on Learning Representations,
                  {ICLR} 2026, Rio de Janeiro, Brazil, April 23-27, 2026},
  publisher    = {proceedings.iclr.cc},
  year         = {2026},
  url          = {https://proceedings.iclr.cc/paper\_files/paper/2026/hash/618c95f4557c15b253fb0e6f548ea0c0-Abstract-Conference.html},
  bibsource    = {dblp computer science bibliography, https://dblp.org}
}

@inproceedings{gu2024,
  author       = {Yuxian Gu and
                  Li Dong and
                  Furu Wei and
                  Minlie Huang},
  title        = {{MiniLLM}: Knowledge Distillation of Large Language Models},
  booktitle    = {The Twelfth International Conference on Learning Representations,
                  {ICLR} 2024, Vienna, Austria, May 7-11, 2024},
  publisher    = {OpenReview.net},
  year         = {2024},
  url          = {https://openreview.net/forum?id=5h0qf7IBZZ},
  bibsource    = {dblp computer science bibliography, https://dblp.org}
}

@inproceedings{qian2025,
  author       = {Cheng Qian and
                  Emre Can Acikgoz and
                  Qi He and
                  Hongru Wang and
                  Xiusi Chen and
                  Dilek Hakkani{-}Tur and
                  Gokhan Tur and
                  Heng Ji},
  editor       = {Danielle Belgrave and
                  Cheng Zhang and
                  Laura N. Montoya and
                  Hsuan{-}Tien Lin and
                  Razvan Pascanu and
                  Piotr Koniusz and
                  Marzyeh Ghassemi and
                  Nancy Chen and
                  Iv{\'{a}}n Vladimir Meza Ru{\'{\i}}z and
                  Arturo Loaiza{-}Bonilla},
  title        = {{ToolRL}: Reward is All Tool Learning Needs},
  booktitle    = {Advances in Neural Information Processing Systems 38: Annual Conference
                  on Neural Information Processing Systems 2025, NeurIPS 2025, San Diego,
                  CA, USA, December 2-7, 2025 / Mexico City, Mexico, November 30 - December
                  5, 2025},
  year         = {2025},
  url          = {http://papers.nips.cc/paper\_files/paper/2025/hash/97c5b2707228e7e3fb67e4ecc2e0e607-Abstract-Conference.html},
  bibsource    = {dblp computer science bibliography, https://dblp.org}
}

@inproceedings{liu2025,
  author       = {Weiwen Liu and
                  Xu Huang and
                  Xingshan Zeng and
                  Xinlong Hao and
                  Shuai Yu and
                  Dexun Li and
                  Shuai Wang and
                  Weinan Gan and
                  Zhengying Liu and
                  Yuanqing Yu and
                  Zezhong Wang and
                  Yuxian Wang and
                  Wu Ning and
                  Yutai Hou and
                  Bin Wang and
                  Chuhan Wu and
                  Xinzhi Wang and
                  Yong Liu and
                  Yasheng Wang and
                  Duyu Tang and
                  Dandan Tu and
                  Lifeng Shang and
                  Xin Jiang and
                  Ruiming Tang and
                  Defu Lian and
                  Qun Liu and
                  Enhong Chen},
  title        = {{ToolACE}: Winning the Points of {LLM} Function Calling},
  booktitle    = {The Thirteenth International Conference on Learning Representations,
                  {ICLR} 2025, Singapore, April 24-28, 2025},
  publisher    = {OpenReview.net},
  year         = {2025},
  url          = {https://openreview.net/forum?id=8EB8k6DdCU},
  bibsource    = {dblp computer science bibliography, https://dblp.org}
}

@inproceedings{lin2025,
  author       = {Qiqiang Lin and
                  Muning Wen and
                  Qiuying Peng and
                  Guanyu Nie and
                  Junwei Liao and
                  Jun Wang and
                  Xiaoyun Mo and
                  Jiamu Zhou and
                  Cheng Cheng and
                  Yin Zhao and
                  Jun Wang and
                  Weinan Zhang},
  title        = {Robust Function-Calling for On-Device Language Model via Function
                  Masking},
  booktitle    = {The Thirteenth International Conference on Learning Representations,
                  {ICLR} 2025, Singapore, April 24-28, 2025},
  publisher    = {OpenReview.net},
  year         = {2025},
  url          = {https://openreview.net/forum?id=yVQcr4qjD6},
  bibsource    = {dblp computer science bibliography, https://dblp.org}
}

@inproceedings{zhang2025,
  author       = {Jianguo Zhang and
                  Tian Lan and
                  Ming Zhu and
                  Zuxin Liu and
                  Thai Hoang and
                  Shirley Kokane and
                  Weiran Yao and
                  Juntao Tan and
                  Akshara Prabhakar and
                  Haolin Chen and
                  Zhiwei Liu and
                  Yihao Feng and
                  Tulika Manoj Awalgaonkar and
                  Rithesh R. N. and
                  Zeyuan Chen and
                  Ran Xu and
                  Juan Carlos Niebles and
                  Shelby Heinecke and
                  Huan Wang and
                  Silvio Savarese and
                  Caiming Xiong},
  editor       = {Luis Chiruzzo and
                  Alan Ritter and
                  Lu Wang},
  title        = {{xLAM}: {A} Family of Large Action Models to Empower {AI} Agent Systems},
  booktitle    = {Proceedings of the 2025 Conference of the Nations of the Americas
                  Chapter of the Association for Computational Linguistics: Human Language
                  Technologies, {NAACL} 2025 - Volume 1: Long Papers, Albuquerque, New
                  Mexico, USA, April 29 - May 4, 2025},
  pages        = {11583--11597},
  publisher    = {Association for Computational Linguistics},
  year         = {2025},
  url          = {https://doi.org/10.18653/v1/2025.naacl-long.578},
  doi          = {10.18653/V1/2025.NAACL-LONG.578},
  bibsource    = {dblp computer science bibliography, https://dblp.org}
}

@article{yang2024,
  author       = {An Yang and
                  Baosong Yang and
                  Beichen Zhang and
                  Binyuan Hui and
                  Bo Zheng and
                  Bowen Yu and
                  Chengyuan Li and
                  Dayiheng Liu and
                  Fei Huang and
                  Haoran Wei and
                  Huan Lin and
                  Jian Yang and
                  Jianhong Tu and
                  Jianwei Zhang and
                  Jianxin Yang and
                  Jiaxi Yang and
                  Jingren Zhou and
                  Junyang Lin and
                  Kai Dang and
                  Keming Lu and
                  Keqin Bao and
                  Kexin Yang and
                  Le Yu and
                  Mei Li and
                  Mingfeng Xue and
                  Pei Zhang and
                  Qin Zhu and
                  Rui Men and
                  Runji Lin and
                  Tianhao Li and
                  Tingyu Xia and
                  Xingzhang Ren and
                  Xuancheng Ren and
                  Yang Fan and
                  Yang Su and
                  Yichang Zhang and
                  Yu Wan and
                  Yuqiong Liu and
                  Zeyu Cui and
                  Zhenru Zhang and
                  Zihan Qiu},
  title        = {{Qwen2.5} Technical Report},
  journal      = {CoRR},
  volume       = {abs/2412.15115},
  year         = {2024},
  url          = {https://doi.org/10.48550/arXiv.2412.15115},
  doi          = {10.48550/ARXIV.2412.15115},
  eprinttype   = {arXiv},
  eprint       = {2412.15115},
  bibsource    = {dblp computer science bibliography, https://dblp.org}
}

@inproceedings{yang2018,
  author       = {Zhilin Yang and
                  Peng Qi and
                  Saizheng Zhang and
                  Yoshua Bengio and
                  William W. Cohen and
                  Ruslan Salakhutdinov and
                  Christopher D. Manning},
  editor       = {Ellen Riloff and
                  David Chiang and
                  Julia Hockenmaier and
                  Jun'ichi Tsujii},
  title        = {{HotpotQA}: {A} Dataset for Diverse, Explainable Multi-hop Question
                  Answering},
  booktitle    = {Proceedings of the 2018 Conference on Empirical Methods in Natural
                  Language Processing, Brussels, Belgium, October 31 - November 4, 2018},
  pages        = {2369--2380},
  publisher    = {Association for Computational Linguistics},
  year         = {2018},
  url          = {https://doi.org/10.18653/v1/d18-1259},
  doi          = {10.18653/V1/D18-1259},
  bibsource    = {dblp computer science bibliography, https://dblp.org}
}

@article{cobbe2021,
  author       = {Karl Cobbe and
                  Vineet Kosaraju and
                  Mohammad Bavarian and
                  Mark Chen and
                  Heewoo Jun and
                  Lukasz Kaiser and
                  Matthias Plappert and
                  Jerry Tworek and
                  Jacob Hilton and
                  Reiichiro Nakano and
                  Christopher Hesse and
                  John Schulman},
  title        = {Training Verifiers to Solve Math Word Problems},
  journal      = {CoRR},
  volume       = {abs/2110.14168},
  year         = {2021},
  url          = {https://arxiv.org/abs/2110.14168},
  eprinttype    = {arXiv},
  eprint       = {2110.14168},
  bibsource    = {dblp computer science bibliography, https://dblp.org}
}

@inproceedings{patil2025,
  author       = {Shishir G. Patil and
                  Huanzhi Mao and
                  Fanjia Yan and
                  Charlie Cheng{-}Jie Ji and
                  Vishnu Suresh and
                  Ion Stoica and
                  Joseph E. Gonzalez},
  editor       = {Aarti Singh and
                  Maryam Fazel and
                  Daniel Hsu and
                  Simon Lacoste{-}Julien and
                  Felix Berkenkamp and
                  Tegan Maharaj and
                  Kiri Wagstaff and
                  Jerry Zhu},
  title        = {The {Berkeley} Function Calling Leaderboard {(BFCL):} From Tool Use
                  to Agentic Evaluation of Large Language Models},
  booktitle    = {Forty-second International Conference on Machine Learning, {ICML}
                  2025, Vancouver, BC, Canada, July 13-19, 2025},
  series       = {Proceedings of Machine Learning Research},
  volume       = {267},
  pages        = {48371--48392},
  publisher    = {{PMLR} / OpenReview.net},
  year         = {2025},
  url          = {https://proceedings.mlr.press/v267/patil25a.html},
  bibsource    = {dblp computer science bibliography, https://dblp.org}
}

@inproceedings{li2023,
  author       = {Minghao Li and
                  Yingxiu Zhao and
                  Bowen Yu and
                  Feifan Song and
                  Hangyu Li and
                  Haiyang Yu and
                  Zhoujun Li and
                  Fei Huang and
                  Yongbin Li},
  editor       = {Houda Bouamor and
                  Juan Pino and
                  Kalika Bali},
  title        = {{API-Bank}: {A} Comprehensive Benchmark for Tool-Augmented {LLMs}},
  booktitle    = {Proceedings of the 2023 Conference on Empirical Methods in Natural
                  Language Processing, {EMNLP} 2023, Singapore, December 6-10, 2023},
  pages        = {3102--3116},
  publisher    = {Association for Computational Linguistics},
  year         = {2023},
  url          = {https://doi.org/10.18653/v1/2023.emnlp-main.187},
  doi          = {10.18653/V1/2023.EMNLP-MAIN.187},
  bibsource    = {dblp computer science bibliography, https://dblp.org}
}

@inproceedings{erdogan2024,
  author       = {Lutfi Eren Erdogan and
                  Nicholas Lee and
                  Siddharth Jha and
                  Sehoon Kim and
                  Ryan Tabrizi and
                  Suhong Moon and
                  Coleman Richard Charles Hooper and
                  Gopala Anumanchipalli and
                  Kurt Keutzer and
                  Amir Gholami},
  editor       = {Delia Iraz{\'{u}} Hern{\'{a}}ndez Far{\'{\i}}as and
                  Tom Hope and
                  Manling Li},
  title        = {{TinyAgent}: Function Calling at the Edge},
  booktitle    = {Proceedings of the 2024 Conference on Empirical Methods in Natural
                  Language Processing: {EMNLP} 2024 - System Demonstrations, Miami,
                  Florida, USA, November 12-16, 2024},
  pages        = {80--88},
  publisher    = {Association for Computational Linguistics},
  year         = {2024},
  url          = {https://doi.org/10.18653/v1/2024.emnlp-demo.9},
  doi          = {10.18653/V1/2024.EMNLP-DEMO.9},
  bibsource    = {dblp computer science bibliography, https://dblp.org}
}

@inproceedings{liu2024,
  author       = {Zuxin Liu and
                  Thai Hoang and
                  Jianguo Zhang and
                  Ming Zhu and
                  Tian Lan and
                  Shirley Kokane and
                  Juntao Tan and
                  Weiran Yao and
                  Zhiwei Liu and
                  Yihao Feng and
                  Rithesh Murthy and
                  Liangwei Yang and
                  Silvio Savarese and
                  Juan Carlos Niebles and
                  Huan Wang and
                  Shelby Heinecke and
                  Caiming Xiong},
  editor       = {Amir Globersons and
                  Lester Mackey and
                  Danielle Belgrave and
                  Angela Fan and
                  Ulrich Paquet and
                  Jakub M. Tomczak and
                  Cheng Zhang},
  title        = {{APIGen}: Automated {PIpeline} for Generating Verifiable and Diverse Function-Calling
                  Datasets},
  booktitle    = {Advances in Neural Information Processing Systems 37: Annual Conference
                  on Neural Information Processing Systems 2024, NeurIPS 2024, Vancouver,
                  BC, Canada, December 10 - 15, 2024},
  year         = {2024},
  url          = {http://papers.nips.cc/paper\_files/paper/2024/hash/61cce86d180b1184949e58939c4f983d-Abstract-Datasets\_and\_Benchmarks\_Track.html},
  bibsource    = {dblp computer science bibliography, https://dblp.org}
}

@inproceedings{prabhakar2025,
  author       = {Akshara Prabhakar and
                  Zuxin Liu and
                  Ming Zhu and
                  Jianguo Zhang and
                  Tulika Manoj Awalgaonkar and
                  Shiyu Wang and
                  Zhiwei Liu and
                  Haolin Chen and
                  Thai Hoang and
                  Juan Carlos Niebles and
                  Shelby Heinecke and
                  Weiran Yao and
                  Huan Wang and
                  Silvio Savarese and
                  Caiming Xiong},
  editor       = {Danielle Belgrave and
                  Cheng Zhang and
                  Laura N. Montoya and
                  Hsuan{-}Tien Lin and
                  Razvan Pascanu and
                  Piotr Koniusz and
                  Marzyeh Ghassemi and
                  Nancy Chen and
                  Iv{\'{a}}n Vladimir Meza Ru{\'{\i}}z and
                  Arturo Loaiza{-}Bonilla},
  title        = {{APIGen-MT}: Agentic Pipeline for Multi-Turn Data Generation via Simulated
                  Agent-Human Interplay},
  booktitle    = {Advances in Neural Information Processing Systems 38: Annual Conference
                  on Neural Information Processing Systems 2025, NeurIPS 2025, San Diego,
                  CA, USA, December 2-7, 2025 / Mexico City, Mexico, November 30 - December
                  5, 2025},
  year         = {2025},
  url          = {http://papers.nips.cc/paper\_files/paper/2025/hash/5e3661f7fe4c8ac5652d62eb3d3c96ea-Abstract-Datasets\_and\_Benchmarks\_Track.html},
  bibsource    = {dblp computer science bibliography, https://dblp.org}
}

@inproceedings{zhang2026a,
  author       = {Shaokun Zhang and
                  Yi Dong and
                  Jieyu Zhang and
                  Jan Kautz and
                  Bryan Catanzaro and
                  Andrew Tao and
                  Qingyun Wu and
                  Zhiding Yu and
                  Guilin Liu},
  editor       = {Carl Vondrick and
                  Bharath Hariharan and
                  Colin Raffel and
                  Lerrel Pinto and
                  Diyi Yang and
                  Aleksandra Faust},
  title        = {{Nemotron-Research-Tool-N1}: Exploring Tool-Using Language Models with
                  Reinforced Reasoning},
  booktitle    = {The Fourteenth International Conference on Learning Representations,
                  {ICLR} 2026, Rio de Janeiro, Brazil, April 23-27, 2026},
  publisher    = {proceedings.iclr.cc},
  year         = {2026},
  url          = {https://proceedings.iclr.cc/paper\_files/paper/2026/hash/93b3d975f9a2448964a906199db98a9d-Abstract-Conference.html},
  bibsource    = {dblp computer science bibliography, https://dblp.org}
}

@inproceedings{li2026,
  author       = {Xiaoxi Li and
                  Wenxiang Jiao and
                  Jiarui Jin and
                  Guanting Dong and
                  Jiajie Jin and
                  Yinuo Wang and
                  Hao Wang and
                  Yutao Zhu and
                  Ji{-}Rong Wen and
                  Yuan Lu and
                  Zhicheng Dou},
  editor       = {Hakim Hacid and
                  Yoelle Maarek and
                  Francesco Bonchi and
                  Ido Guy and
                  Emine Yilmaz},
  title        = {{DeepAgent}: {A} General Reasoning Agent with Scalable Toolsets},
  booktitle    = {Proceedings of the {ACM} Web Conference 2026, {WWW} 2026, Dubai, United
                  Arab Emirates, originally scheduled for April 13-17, 2026, rescheduled
                  for June 29 - July 3, 2026},
  pages        = {2219--2230},
  publisher    = {{ACM}},
  year         = {2026},
  url          = {https://doi.org/10.1145/3774904.3792460},
  doi          = {10.1145/3774904.3792460},
  bibsource    = {dblp computer science bibliography, https://dblp.org}
}

@article{acikgoz2026,
  author       = {Emre Can Acikgoz and
                  Cheng Qian and
                  Jonas H{\"{u}}botter and
                  Heng Ji and
                  Dilek Hakkani{-}T{\"{u}}r and
                  Gokhan Tur},
  title        = {{Tool-R0}: Self-Evolving {LLM} Agents for Tool-Learning from Zero Data},
  journal      = {CoRR},
  volume       = {abs/2602.21320},
  year         = {2026},
  url          = {https://doi.org/10.48550/arXiv.2602.21320},
  doi          = {10.48550/ARXIV.2602.21320},
  eprinttype   = {arXiv},
  eprint       = {2602.21320},
  bibsource    = {dblp computer science bibliography, https://dblp.org}
}

@article{gao2026,
  author       = {Xingjie Gao and
                  Pengcheng Huang and
                  Zhenghao Liu and
                  Yukun Yan and
                  Shuo Wang and
                  Zulong Chen and
                  Chen Qian and
                  Ge Yu and
                  Yu Gu},
  title        = {Teaching {LLMs} to Learn Tool Trialing and Execution through Environment
                  Interaction},
  journal      = {CoRR},
  volume       = {abs/2601.12762},
  year         = {2026},
  url          = {https://doi.org/10.48550/arXiv.2601.12762},
  doi          = {10.48550/ARXIV.2601.12762},
  eprinttype   = {arXiv},
  eprint       = {2601.12762},
  bibsource    = {dblp computer science bibliography, https://dblp.org}
}

@inproceedings{zhang2026b,
  author       = {Kangning Zhang and
                  Weiwen Liu and
                  Wenxiang Jiao and
                  Kounianhua Du and
                  Yuan Lu and
                  Weinan Zhang and
                  Yong Yu},
  editor       = {Maria Liakata and
                  Viviane P. Moreira and
                  Jiajun Zhang and
                  David Jurgens},
  title        = {{LoopTool}: Closing the Data-Training Loop for Robust {LLM} Tool Calls},
  booktitle    = {Proceedings of the 64th Annual Meeting of the Association for Computational
                  Linguistics (Volume 1: Long Papers), {ACL} 2026, San Diego, California,
                  United States, July 2-7, 2026},
  pages        = {21144--21164},
  publisher    = {Association for Computational Linguistics},
  year         = {2026},
  url          = {https://doi.org/10.18653/v1/2026.acl-long.968},
  doi          = {10.18653/V1/2026.ACL-LONG.968},
  bibsource    = {dblp computer science bibliography, https://dblp.org}
}

@inproceedings{kang2025,
  author       = {Minki Kang and
                  Jongwon Jeong and
                  Seanie Lee and
                  Jaewoong Cho and
                  Sung Ju Hwang},
  editor       = {Danielle Belgrave and
                  Cheng Zhang and
                  Laura N. Montoya and
                  Hsuan{-}Tien Lin and
                  Razvan Pascanu and
                  Piotr Koniusz and
                  Marzyeh Ghassemi and
                  Nancy Chen and
                  Iv{\'{a}}n Vladimir Meza Ru{\'{\i}}z and
                  Arturo Loaiza{-}Bonilla},
  title        = {Distilling {LLM} Agent into Small Models with Retrieval and Code Tools},
  booktitle    = {Advances in Neural Information Processing Systems 38: Annual Conference
                  on Neural Information Processing Systems 2025, NeurIPS 2025, San Diego,
                  CA, USA, December 2-7, 2025 / Mexico City, Mexico, November 30 - December
                  5, 2025},
  year         = {2025},
  url          = {http://papers.nips.cc/paper\_files/paper/2025/hash/99263f9bf46874e2b8b7f104b5063864-Abstract-Conference.html},
  bibsource    = {dblp computer science bibliography, https://dblp.org}
}

@inproceedings{shing2025,
  author       = {Makoto Shing and
                  Kou Misaki and
                  Han Bao and
                  Sho Yokoi and
                  Takuya Akiba},
  title        = {{TAID:} Temporally Adaptive Interpolated Distillation for Efficient
                  Knowledge Transfer in Language Models},
  booktitle    = {The Thirteenth International Conference on Learning Representations,
                  {ICLR} 2025, Singapore, April 24-28, 2025},
  publisher    = {OpenReview.net},
  year         = {2025},
  url          = {https://openreview.net/forum?id=cqsw28DuMW},
  bibsource    = {dblp computer science bibliography, https://dblp.org}
}

@inproceedings{tang2026,
  author       = {Guokai Tang and
                  Feng Zhao},
  editor       = {Maria Liakata and
                  Viviane P. Moreira and
                  Jiajun Zhang and
                  David Jurgens},
  title        = {{SmartAD}: Capacity-Aligned Agent Distillation for Small Language Models},
  booktitle    = {Findings of the Association for Computational Linguistics, {ACL} 2026,
                  San Diego, California, United States, July 2-7, 2026},
  pages        = {27045--27057},
  publisher    = {Association for Computational Linguistics},
  year         = {2026},
  url          = {https://doi.org/10.18653/v1/2026.findings-acl.1349},
  doi          = {10.18653/V1/2026.FINDINGS-ACL.1349},
  bibsource    = {dblp computer science bibliography, https://dblp.org}
}

@inproceedings{ni2026,
  author       = {Jiliang Ni and
                  Jiachen Pu and
                  Zhongyi Yang and
                  Jingfeng Luo and
                  Conggang Hu},
  editor       = {Carl Vondrick and
                  Bharath Hariharan and
                  Colin Raffel and
                  Lerrel Pinto and
                  Diyi Yang and
                  Aleksandra Faust},
  title        = {{STAR:} Similarity-guided Teacher-Assisted Refinement for Super-Tiny
                  Function Calling Models},
  booktitle    = {The Fourteenth International Conference on Learning Representations,
                  {ICLR} 2026, Rio de Janeiro, Brazil, April 23-27, 2026},
  publisher    = {proceedings.iclr.cc},
  year         = {2026},
  url          = {https://proceedings.iclr.cc/paper\_files/paper/2026/hash/6c8763fec7634dbaf69d299ddf96687f-Abstract-Conference.html},
  bibsource    = {dblp computer science bibliography, https://dblp.org}
}

@inproceedings{yan2025,
  author       = {Jianhao Yan and
                  Yafu Li and
                  Zican Hu and
                  Zhi Wang and
                  Ganqu Cui and
                  Xiaoye Qu and
                  Yu Cheng and
                  Yue Zhang},
  editor       = {Danielle Belgrave and
                  Cheng Zhang and
                  Laura N. Montoya and
                  Hsuan{-}Tien Lin and
                  Razvan Pascanu and
                  Piotr Koniusz and
                  Marzyeh Ghassemi and
                  Nancy Chen and
                  Iv{\'{a}}n Vladimir Meza Ru{\'{\i}}z and
                  Arturo Loaiza{-}Bonilla},
  title        = {Learning to Reason under Off-Policy Guidance},
  booktitle    = {Advances in Neural Information Processing Systems 38: Annual Conference
                  on Neural Information Processing Systems 2025, NeurIPS 2025, San Diego,
                  CA, USA, December 2-7, 2025 / Mexico City, Mexico, November 30 - December
                  5, 2025},
  year         = {2025},
  url          = {http://papers.nips.cc/paper\_files/paper/2025/hash/a9d5c33e38442450f9781bd7bad3c81e-Abstract-Conference.html},
  bibsource    = {dblp computer science bibliography, https://dblp.org}
}

@inproceedings{ma2026,
  author       = {Lu Ma and
                  Hao Liang and
                  Meiyi Qiang and
                  Lexiang Tang and
                  Xiaochen Ma and
                  Zhen Hao Wong and
                  Junbo Niu and
                  Chengyu Shen and
                  Runming He and
                  Yanhao Li and
                  Bin Cui and
                  Wentao Zhang},
  editor       = {Carl Vondrick and
                  Bharath Hariharan and
                  Colin Raffel and
                  Lerrel Pinto and
                  Diyi Yang and
                  Aleksandra Faust},
  title        = {Learning What Reinforcement Learning Can't: Interleaved Online Fine-Tuning
                  for Hardest Questions},
  booktitle    = {The Fourteenth International Conference on Learning Representations,
                  {ICLR} 2026, Rio de Janeiro, Brazil, April 23-27, 2026},
  publisher    = {proceedings.iclr.cc},
  year         = {2026},
  url          = {https://proceedings.iclr.cc/paper\_files/paper/2026/hash/82c4ad0c3d3d4898a77be5df7b2ede3a-Abstract-Conference.html},
  bibsource    = {dblp computer science bibliography, https://dblp.org}
}

@inproceedings{guo2026,
  author       = {Jiaxin Guo and
                  Hao Sun and
                  Wenhao Zhang and
                  Chunyu Yang and
                  Yan Zhang},
  editor       = {Maria Liakata and
                  Viviane P. Moreira and
                  Jiajun Zhang and
                  David Jurgens},
  title        = {Failures are Treasures: Constructing a Pedagogical Bridge for Agentic
                  Strategy Distillation},
  booktitle    = {Findings of the Association for Computational Linguistics, {ACL} 2026,
                  San Diego, California, United States, July 2-7, 2026},
  pages        = {18808--18823},
  publisher    = {Association for Computational Linguistics},
  year         = {2026},
  url          = {https://doi.org/10.18653/v1/2026.findings-acl.938},
  doi          = {10.18653/V1/2026.FINDINGS-ACL.938},
  bibsource    = {dblp computer science bibliography, https://dblp.org}
}

@article{zhong2026,
  author       = {Qiyong Zhong and
                  Mao Zheng and
                  Mingyang Song and
                  Xin Lin and
                  Jie Sun and
                  Houcheng Jiang and
                  Xiang Wang and
                  Junfeng Fang},
  title        = {{SOD:} Step-wise On-policy Distillation for Small Language Model Agents},
  journal      = {CoRR},
  volume       = {abs/2605.07725},
  year         = {2026},
  url          = {https://doi.org/10.48550/arXiv.2605.07725},
  doi          = {10.48550/ARXIV.2605.07725},
  eprinttype   = {arXiv},
  eprint       = {2605.07725},
  bibsource    = {dblp computer science bibliography, https://dblp.org}
}

@article{xu2025a,
  author       = {Hongling Xu and
                  Qi Zhu and
                  Heyuan Deng and
                  Jinpeng Li and
                  Lu Hou and
                  Yasheng Wang and
                  Lifeng Shang and
                  Ruifeng Xu and
                  Fei Mi},
  title        = {{KDRL:} Post-Training Reasoning {LLMs} via Unified Knowledge Distillation
                  and Reinforcement Learning},
  journal      = {CoRR},
  volume       = {abs/2506.02208},
  year         = {2025},
  url          = {https://doi.org/10.48550/arXiv.2506.02208},
  doi          = {10.48550/ARXIV.2506.02208},
  eprinttype   = {arXiv},
  eprint       = {2506.02208},
  bibsource    = {dblp computer science bibliography, https://dblp.org}
}

@inproceedings{kotoge2026,
  author       = {Rikuto Kotoge and
                  Mai Nishimura and
                  Jiaxin Ma},
  editor       = {Maria Liakata and
                  Viviane P. Moreira and
                  Jiajun Zhang and
                  David Jurgens},
  title        = {Can Compact Language Models Search Like Agents? Distillation-Guided
                  Policy Optimization for Preserving Agentic {RAG} Capabilities},
  booktitle    = {Proceedings of the 64th Annual Meeting of the Association for Computational
                  Linguistics (Volume 1: Long Papers), {ACL} 2026, San Diego, California,
                  United States, July 2-7, 2026},
  pages        = {37733--37746},
  publisher    = {Association for Computational Linguistics},
  year         = {2026},
  url          = {https://doi.org/10.18653/v1/2026.acl-long.1751},
  doi          = {10.18653/V1/2026.ACL-LONG.1751},
  bibsource    = {dblp computer science bibliography, https://dblp.org}
}

@inproceedings{liu2026,
  author       = {Jiazhen Liu and
                  Yuchuan Deng and
                  Long Chen},
  editor       = {Carl Vondrick and
                  Bharath Hariharan and
                  Colin Raffel and
                  Lerrel Pinto and
                  Diyi Yang and
                  Aleksandra Faust},
  title        = {Empowering Small {VLMs} to Think with Dynamic Memorization and Exploration},
  booktitle    = {The Fourteenth International Conference on Learning Representations,
                  {ICLR} 2026, Rio de Janeiro, Brazil, April 23-27, 2026},
  publisher    = {proceedings.iclr.cc},
  year         = {2026},
  url          = {https://proceedings.iclr.cc/paper\_files/paper/2026/hash/82b0c1b954b6ef9f3cfb664a82b201bb-Abstract-Conference.html},
  bibsource    = {dblp computer science bibliography, https://dblp.org}
}

@article{li2026a,
  author       = {Gengsheng Li and
                  Tianyu Yang and
                  Junfeng Fang and
                  Mingyang Song and
                  Mao Zheng and
                  Haiyun Guo and
                  Dan Zhang and
                  Jinqiao Wang and
                  Tat{-}Seng Chua},
  title        = {Unifying Group-Relative and Self-Distillation Policy Optimization
                  via Sample Routing},
  journal      = {CoRR},
  volume       = {abs/2604.02288},
  year         = {2026},
  url          = {https://doi.org/10.48550/arXiv.2604.02288},
  doi          = {10.48550/ARXIV.2604.02288},
  eprinttype   = {arXiv},
  eprint       = {2604.02288},
  bibsource    = {dblp computer science bibliography, https://dblp.org}
}

@article{han2026,
  author       = {Zhuowen Han and
                  Jinwei Xiao and
                  Zhengxi Lu and
                  Renren Jin and
                  Zhiyuan Yao and
                  Yuxin Liu and
                  Hongyan Hao and
                  Yueqing Sun and
                  Yu Yang and
                  Qi Gu and
                  Xunliang Cai and
                  Deyi Xiong},
  title        = {Distill Where You Fail: Recovering Learning Signals of Negative {RL}-Groups
                  from Adaptive Teacher Guidance},
  journal      = {CoRR},
  volume       = {abs/2608.00782},
  year         = {2026},
  url          = {https://doi.org/10.48550/arXiv.2608.00782},
  doi          = {10.48550/ARXIV.2608.00782},
  eprinttype   = {arXiv},
  eprint       = {2608.00782},
  bibsource    = {dblp computer science bibliography, https://dblp.org}
}
\appendix
\flushbottom
\AddToHookNext{shipout/after}{\begingroup\globaldefs=1\raggedbottom\endgroup}
\section{Implementation and Reproducibility Details}
\label{sec:implementation_details}

\subsection{Models, Data, and Hardware}
\label{sec:models_and_data}
\label{sec:environment}

All students and teachers use the Qwen2.5-Instruct series~\citep{yang2024}. The 0.5B, 1.5B, and 3B students are paired with 1.5B, 7B, and 14B teachers trained with \ours, GRPO, and SFT+GRPO, respectively. In the teacher-scale comparison, the 0.5B teacher is trained with \ours, and the 3B teacher follows the same SFT+GRPO recipe as the 14B teacher. Teacher parameters remain frozen during student policy optimization.

The cold-start corpus contains 400 expert reasoning and tool-use demonstrations, split into 360 training and 40 validation instances. The subsequent training corpus contains 2,000 ToolACE~\citep{liu2025}, 1,000 Hammer~\citep{lin2025}, and 1,000 xLAM~\citep{zhang2025} instances, split into 3,920 training and 80 validation instances.

The 0.5B and 1.5B experiments use one and two NVIDIA L40S GPUs, respectively. We use two NVIDIA RTX PRO 6000 Blackwell GPUs for the 3B student and its 14B teacher.

\subsection{Training Configuration}
\label{sec:training_configuration}
\label{sec:cold_start_details}
\label{sec:offline_distillation}
\label{sec:setpo_training}

LoRA-based cold-start SFT initializes student policy optimization, using the settings in Table~\ref{tab:sft_hyperparameters}.

\begin{table}[htbp]
\centering
\small
\begin{tabular}{lc}
\toprule
\textbf{Hyperparameter} & \textbf{Value} \\
\midrule
Learning Rate & $2 \times 10^{-5}$ \\
LR Scheduler & Cosine \\
Total Epochs & 3 \\
Global Batch Size & 32 \\
Micro Batch Size / GPU & 4 \\
Max Sequence Length & 4096 \\
\bottomrule
\end{tabular}
\caption{Cold-start initialization settings.}
\label{tab:sft_hyperparameters}
\end{table}

Before RL, the corresponding frozen teacher generates one reference trajectory for each training prompt using its tokenizer and chat template, greedy decoding, and a maximum output length of 1,024 tokens. These trajectories form the auxiliary SFT dataset. During RL, online teacher feedback uses student-generated prefixes, whereas auxiliary SFT uses complete teacher trajectories and their prefixes. Table~\ref{tab:hyperparameters} gives the policy-optimization settings. We enable gradient checkpointing and report results from the final trained student model.

\begin{table}[htbp]
\centering
\small
\begin{tabular}{lc}
\toprule
\textbf{Hyperparameter} & \textbf{Value} \\
\midrule
Precision & Bfloat16 \\
LoRA Rank ($r$) & 32 \\
LoRA Alpha ($\alpha$) & 64 \\
Actor Learning Rate & $2 \times 10^{-5}$ \\
GRPO Group Size ($G$) & 4 \\
RL Total Epochs & 15 \\
Optimization Steps & 105 \\
Total Train Batch Size & 512 \\
GRPO Mini Batch Size & 64 \\
GRPO Micro Batch Size / GPU & 8 \\
Intervention Threshold ($\gamma$) & 3.99 \\
Teacher-Guidance Weight ($\beta$) & 1.0 \\
Auxiliary SFT Weight ($\lambda$) & 0.1 \\
Max Prompt Length & 2048 \\
Max Response Length & 1024 \\
Max Teacher Scoring Length & 1024 \\
\bottomrule
\end{tabular}
\caption{\ours{} policy-optimization settings.}
\label{tab:hyperparameters}
\end{table}

\subsection{Baselines and Evaluation}
\label{sec:baseline_details}
\label{sec:evaluation_details}

\paragraph{Main Baselines.}
Vanilla is the original Qwen2.5-Instruct model. SFT learns the provided reference responses and tool-use trajectories from the subsequent training corpus. GRPO starts directly from Vanilla, whereas SFT+GRPO uses the cold-start initialization in Section~\ref{sec:training_configuration}; neither uses online teacher guidance or auxiliary SFT during policy optimization. Both follow the main RL data, group size, reward function, and training schedule. Offline KD learns only pre-generated teacher trajectories, while OPD trains on frozen-teacher feedback for student-generated trajectories. The additional 0.5B GRPO+OPD baseline jointly optimizes task-reward GRPO and a separate on-policy distillation loss evaluated under student-generated prefixes: the distillation loss is added to the objective without changing the task-reward GRPO advantages. The distillation loss applies to all sampled trajectories without reward-based gating.

\paragraph{Intervention Controls.}
Always-on guidance applies the same teacher signal as \ours{} to every sampled trajectory. Random intervention selects, at each update, the same number of trajectories as the reward gate would select, matching the teacher-query budget. Both change only which trajectories receive guidance, keeping the other training components and settings unchanged.

\paragraph{Component Ablations.}
Trajectory-level distillation averages $d_{j,t}$ over valid generated tokens in each selected trajectory and shares this correction across its tokens. Without SFT Loss sets $\lambda=0$, retaining task-reward GRPO and selective online guidance. Without Online Teacher retains task-reward GRPO and persistent supervision from complete teacher demonstrations.

API-Bank~\citep{li2023} uses 597 processed examples, with 399, 67, and 131 at Levels 1, 2, and 3, scored by the project's soft-match evaluator. Overall accuracy is computed across all examples, rather than by averaging level accuracies. We report one training run per configuration. General reasoning is evaluated on 1,319 GSM8K test examples~\citep{cobbe2021} and a 100-question HotpotQA subset~\citep{yang2018}.

\subsection{Reward Specification}
\label{sec:reward_details}

The deterministic checker assigns a format reward of 1 when required structural tags, including \texttt{<think>} and \texttt{<tool\_call>}, are present in the correct order, and 0 otherwise. Correctness compares predicted calls $P=\{P_1,\ldots,P_m\}$ with reference calls $G=\{G_1,\ldots,G_n\}$. Let $N_P$ and $N_G$ be their tool-name sets, and let $P_i$ denote the predicted call matched to $G_j$. Tool-name, parameter-name, and parameter-value scores are
\begin{align}
r_{\mathrm{name}}
&=\frac{|N_G\cap N_P|}{|N_G\cup N_P|},\notag\\\displaybreak[1]
r_{\mathrm{param}}
&=\sum_{G_j\in G}
\frac{|\operatorname{keys}(G_j)\cap\operatorname{keys}(P_i)|}
{|\operatorname{keys}(G_j)\cup\operatorname{keys}(P_i)|},\notag\\\displaybreak[1]
r_{\mathrm{value}}
&=\sum_{G_j\in G}\sum_{k\in\operatorname{keys}(G_j)}
\mathbb I[G_j[k]=P_i[k]].
\end{align}
Tool names and parameter names use Jaccard matching, while parameter values use exact matching. The normalization and correctness reward are
\begin{equation}
\begin{aligned}
S_{\max}
&=1+|G|+\sum_{G_j\in G}|\operatorname{keys}(G_j)|,\\
R_{\mathrm{correct}}
&=6\frac{r_{\mathrm{name}}+r_{\mathrm{param}}+r_{\mathrm{value}}}{S_{\max}}-3.
\end{aligned}
\end{equation}
Correctness rewards range from $-3$ to $3$, giving a maximum total reward of 4 when combined with the format reward. Equations~\ref{eq:final_reward} and~\ref{eq:intervention_indicator} define total rewards and online intervention; Table~\ref{tab:hyperparameters} lists the default intervention threshold.

\section{Proof of Theorem~\ref{thm:teacher_preservation}}
\label{app:teacher_preservation}
\begin{proof}
Let $\Delta(\mathcal{Y})$ denote the probability simplex
over $\mathcal{Y}$. For brevity, write
$p_y=p(y)$, $q_y=q(y)$, and $R_y=R(y)$.
All logarithms are natural.
\par\medskip
\noindent
\textbf{Step 1: Existence and uniqueness.}
Under Assumption~\ref{ass:population}, $\mathcal{Y}$
is finite and $q_y>0$ for every $y\in\mathcal{Y}$.
The objective $\mathcal{F}$ is finite at $p=q$
and continuous in the interior of
$\Delta(\mathcal{Y})$.
As $p$ approaches the boundary of the simplex,
at least one coordinate $p_y$ approaches zero.
Since $q_y>0$, this implies
\[
D_{\mathrm{KL}}(q\|p)\longrightarrow+\infty.
\]
The reward is bounded, and both KL divergences
are nonnegative. Therefore,
\[
\mathcal{F}(p)\longrightarrow-\infty
\]
as $p$ approaches the boundary.
By compactness of the simplex, the maximum is
attained at an interior distribution.
For any interior $p$, the Hessian is
\begin{equation}
\label{eq:app_objective_hessian}
\nabla^2\mathcal{F}(p)
=
-\operatorname{diag}_{y\in\mathcal{Y}}
\left(
\frac{a}{p_y}
+
\frac{bq_y}{p_y^2}
\right)
\prec 0.
\end{equation}
Hence, $\mathcal{F}$ is strictly concave on the
interior of the simplex, and its maximizer
$p^\star$ is unique.
\par\medskip
\noindent
\textbf{Step 2: Expected reward improvement.}
The teacher distribution $q$ is feasible and
both regularization terms vanish at $p=q$.
Consequently,
\begin{equation}
\label{eq:app_teacher_comparison}
\mathcal{F}(p^\star)
\geq
\mathcal{F}(q)
=
\mathbb{E}_q[R].
\end{equation}
Expanding $\mathcal{F}(p^\star)$ and rearranging
gives
\begin{equation}
\label{eq:app_reward_improvement}
\begin{aligned}
\mathbb{E}_{p^\star}[R]-\mathbb{E}_q[R]
&\geq
aD_{\mathrm{KL}}(p^\star\|q) \\
&\quad+
bD_{\mathrm{KL}}(q\|p^\star) \\
&\geq 0.
\end{aligned}
\end{equation}
Thus, the optimal distribution achieves an
expected reward no lower than that of the teacher.
\par\medskip
\noindent
\textbf{Step 3: Characterization of the optimal ratios.}
Introduce a Lagrange multiplier $\eta$ for
the normalization constraint $\sum_y p_y=1$.
Since $p^\star$ is interior, the first-order
optimality condition is
\begin{equation}
\label{eq:app_stationarity}
R_y
-a\left(
\log\frac{p_y^\star}{q_y}+1
\right)
+b\frac{q_y}{p_y^\star}
-\eta
=0.
\end{equation}
Define
\[
u_y=\frac{p_y^\star}{q_y},
\qquad
\nu=\eta+a-b,
\]
and
\[
h(u)=a\log u+b\left(1-\frac{1}{u}\right).
\]
Equation~\eqref{eq:app_stationarity} becomes
\begin{equation}
\label{eq:app_ratio_characterization}
R_y-\nu=h(u_y).
\end{equation}
For every $u>0$,
\[
h'(u)=\frac{a}{u}+\frac{b}{u^2}>0,
\qquad h(1)=0.
\]
Thus, $h$ is strictly increasing.
Normalization yields
\begin{equation}
\label{eq:app_ratio_normalization}
\sum_{y\in\mathcal{Y}}q_yu_y=1.
\end{equation}
Because every $q_y$ is positive, at least one
ratio $u_y$ is at most one and at least one is
at least one. Combining this fact with
Eq.~\eqref{eq:app_ratio_characterization}
and $h(1)=0$ gives
\begin{equation}
\label{eq:app_multiplier_range}
\min_{y\in\mathcal{Y}}R_y
\leq \nu \leq
\max_{y\in\mathcal{Y}}R_y.
\end{equation}
Let
\[
\Delta_R
=
\max_{y\in\mathcal{Y}}R_y
-
\min_{y\in\mathcal{Y}}R_y.
\]
It follows that, for every response $y$,
\begin{equation}
\label{eq:app_h_range}
-\Delta_R
\leq R_y-\nu=h(u_y)
\leq \Delta_R.
\end{equation}
\par\medskip
\noindent
\textbf{Step 4: Probability-ratio bounds.}
First consider a response with $u_y\leq 1$.
By Eq.~\eqref{eq:app_h_range},
\begin{equation}
\label{eq:app_lower_ratio_intermediate}
\begin{aligned}
a\log\frac{1}{u_y}
+
b\left(\frac{1}{u_y}-1\right)
&=
\nu-R_y \\
&\leq \Delta_R.
\end{aligned}
\end{equation}
Both terms on the left-hand side are nonnegative.
Discarding the first term gives
\[
b\left(\frac{1}{u_y}-1\right)
\leq \Delta_R,
\]
and therefore
\begin{equation}
\label{eq:app_lower_ratio_bound}
u_y\geq\frac{b}{b+\Delta_R}.
\end{equation}
For $u_y>1$, the same lower bound holds
automatically because $b/(b+\Delta_R)\leq 1$.
Next consider a response with $u_y\geq 1$.
Since $b(1-u_y^{-1})\geq 0$, we obtain
\[
a\log u_y
\leq h(u_y)
\leq \Delta_R.
\]
Hence,
\begin{equation}
\label{eq:app_upper_ratio_bound}
u_y
\leq
\exp\left(\frac{\Delta_R}{a}\right).
\end{equation}
For $u_y<1$, this upper bound holds automatically
because $\exp(\Delta_R/a)\geq 1$.
Combining the two bounds proves that
\begin{equation}
\label{eq:app_probability_ratio_bound}
\frac{b}{b+\Delta_R}
\leq
\frac{p^\star(y)}{q(y)}
\leq
\exp\left(\frac{\Delta_R}{a}\right)
\end{equation}
for every $y\in\mathcal{Y}$.
\par\medskip
\noindent
\textbf{Step 5: Preservation of teacher behavior sets.}
For any response set
$\mathcal{E}\subseteq\mathcal{Y}$,
summing the lower bound in
Eq.~\eqref{eq:app_probability_ratio_bound} gives
\begin{equation}
\label{eq:app_behavior_preservation}
\begin{aligned}
p^\star(\mathcal{E})
&=
\sum_{y\in\mathcal{E}}q_yu_y \\
&\geq
\frac{b}{b+\Delta_R}
\sum_{y\in\mathcal{E}}q_y \\
&=
\frac{b}{b+\Delta_R}\,q(\mathcal{E}).
\end{aligned}
\end{equation}
This completes the proof.
\end{proof}

\section{Hyperparameter Analysis}
\label{sec:hyperparameter_sensitivity}
\label{sec:gamma_search}

The experiments in Table~\ref{tab:hyperparameter_sensitivity} use the Qwen2.5-0.5B student and vary the intervention threshold $\gamma$, online teacher-guidance weight $\beta$, or auxiliary SFT weight $\lambda$, holding all other settings fixed.

\begin{table}[htbp]
\centering
\small
\begin{tabular*}{\columnwidth}{@{\extracolsep{\fill}}cccc@{}}
\toprule
\textbf{Parameter} & \textbf{Value} & \textbf{BFCL-v4} & \textbf{API-Bank} \\
\midrule
\multirow{4}{*}{$\gamma$}
& 3.0 & 6.19 & 35.18 \\
& 3.5 & 17.34 & 43.55 \\
& 3.9 & 17.66 & 43.07 \\
& 3.99 & \textbf{18.99} & \textbf{46.06} \\
\midrule
\multirow{3}{*}{$\beta$}
& 0.5 & 17.69 & 45.39 \\
& 1.0 & \textbf{18.99} & \textbf{46.06} \\
& 2.0 & 15.58 & 45.06 \\
\midrule
\multirow{3}{*}{$\lambda$}
& 0.0 & 16.05 & 37.35 \\
& 0.1 & \textbf{18.99} & \textbf{46.06} \\
& 0.2 & 16.21 & 45.56 \\
\bottomrule
\end{tabular*}
\caption{Hyperparameter sensitivity for the Qwen2.5-0.5B student. The best result in each column within each parameter group is bolded.}
\label{tab:hyperparameter_sensitivity}
\label{tab:gamma_sensitivity}
\label{tab:beta_sensitivity}
\label{tab:lambda_sensitivity}
\end{table}

Within the tested range, an intervention threshold close to the maximum reward performs best, favoring guidance for a broader range of sub-optimal attempts; a substantially lower threshold reduces performance. Moderate online guidance is more effective, whereas stronger teacher signals may weaken task-driven updates. Auxiliary SFT also favors a moderate weight: removing it reduces both benchmark scores, particularly API-Bank, whereas increasing its weight gives no further gains.

\makeatletter
\setlength{\@fptop}{0pt}
\makeatother
\section{Teacher Selection Results}
\label{sec:teacher_selection_results}

Table~\ref{tab:teacher_scale_results} reports the Figure~\ref{fig:ablation_teacher} scores for the Qwen2.5-0.5B student across teacher sizes.

\begin{table}[!t]
\centering
\small
\setlength{\tabcolsep}{4pt}
\resizebox{\columnwidth}{!}{%
\begin{tabular}{@{}lcccc@{}}
\toprule
\textbf{Teacher Model}
& \textbf{BFCL-v4}
& \textbf{API-Bank}
& \textbf{GSM8K}
& \textbf{HotpotQA} \\
\midrule
Qwen2.5-0.5B & 16.73 & 38.53 & 30.55 & 23.0 \\
Qwen2.5-1.5B & 18.99 & 46.06 & 38.51 & 35.0 \\
Qwen2.5-3B & 17.52 & 42.04 & 31.77 & 35.0 \\
Qwen2.5-7B & 12.93 & 31.16 & 35.10 & 29.0 \\
\bottomrule
\end{tabular}}
\caption{Performance of the Qwen2.5-0.5B student under teachers of different model scales.}
\label{tab:teacher_scale_results}
\end{table}

\end{document}